\documentclass[10pt,onecolumn,letterpaper]{article}

\usepackage{multirow}
\usepackage{hhline}

\usepackage{cvpr}
\usepackage{times}
\usepackage{epsfig}
\usepackage{graphicx}
\usepackage{amsthm}
\usepackage{amsmath,bm}
\usepackage{amssymb}
\usepackage{algorithm,algpseudocode}
\usepackage{booktabs}
\usepackage{subfigure}
\usepackage{epsfig}
\usepackage{wrapfig}
\usepackage[breaklinks=true,
bookmarks=ture,citecolor=red,    
colorlinks=true,
linkcolor=blue,
filecolor=magenta,      
urlcolor=cyan,]{hyperref}

\allowdisplaybreaks
\DeclareMathOperator{\argmin}{argmin}
\DeclareMathOperator{\argmax}{argmax}
\newtheorem{theorem}{Theorem}[section]
\newtheorem{assumption}{Assumption}[section]
\newtheorem{lemma}{Lemma}[section]
\newtheorem{Def}{Definition}[section]

\newtheorem{remark}{Remark}[section]

\newtheorem*{theorem2}{Theorem}

\numberwithin{equation}{section}
\usepackage{datetime}

\def\r1{\color{black}}
\def\b1{\color{blue}}

\usepackage{fancyhdr}
\usepackage{authblk}
\cvprfinalcopy % *** Uncomment this line for the final submission
\def\cvprPaperID{****} % *** Enter the CVPR Paper ID here
\allowdisplaybreaks

\begin{document}
\lhead{}
%\lfoot{\date{\today},\date{\currenttime}}
%\rfoot{SZVR-G}

\title{Stochastic Zeroth-order  Optimization via Variance Reduction method}
\author[1]{ Liu Liu\thanks{lliu8101@uni.sydney.edu.au}}
\author[2]{ Minhao Cheng\thanks{mhcheng@ucdavis.edu}}
\author[2]{	Cho-Jui Hsieh\thanks{chohsieh@ucdavis.edu}}
\author[1]{ Dacheng Tao\thanks{dacheng.tao@sydney.edu.au}}
\affil[1]{UBTECH Sydney AI Centre and SIT, FEIT, The University of Sydney}
\affil[2]{University of California, Davis}
\maketitle
\maketitle
\thispagestyle{fancy}

\begin{abstract}
	Derivative-free optimization has become an important technique used in 
	machine learning for optimizing black-box models. 
	To conduct updates without explicitly computing gradient, 
	most current approaches iteratively sample a random search direction from 
	Gaussian distribution
	and compute the estimated gradient along that direction.
	However, due to the variance in the search direction, 
	the convergence rates and query complexities of existing methods 
	suffer from a factor of $d$, where $d$ is the problem dimension. 
	In this paper, we introduce a novel Stochastic Zeroth-order method with 
	Variance Reduction under Gaussian smoothing (SZVR-G) and establish the 
	complexity for optimizing non-convex problems. With variance reduction on 
	both sample space and search space, the complexity of 
	our algorithm is sublinear to $d$ and is strictly better than current 
	approaches, in both smooth and non-smooth cases. 
	%  In this paper, we study derivative-free optimization methods that use 
	%only  
	%  function value evaluation rather than the gradient and apply the random 
	%  Gaussian vectors for the search direction. However, the convergence rate 
	%of 
	%  computing the stationary point suffers 
	%  a factor with respect to $d$ over traditional method, where $d$ is the 
	%  problem dimension. We introduce the stochastic zero-order optimization 
	%via 
	%  variance 
	%  reduction method and establish the complexity for the non-convex 
	%problem, 
	%  which 
	%  outperform the 
	%  classical method and  recent randomized stochastic gradient 
	%  method. We also give the complexity analysis of non-smooth function that 
	%is 
	%  better than previous.
	Moreover, we extend the proposed method to 
	the mini-batch version. Our experimental results demonstrate the superior performance of the proposed method over
	existing derivative-free optimization techniques. 
	Furthermore, we successfully apply our 
	method to conduct a universal black-box attack to deep neural networks and 
	present some interesting results. 
\end{abstract}

\section{Introduction}
Derivative-free optimization methods have a long history in optimization 
\cite{matyas1965random}. They use only function value information rather than
explicit gradient calculation to optimize a function, as in the case of 
black-box setting or when computing the partial derivative is too expensive. 
Recently,  derivative-free 
methods received 
substantial attention  in machine 
learning and deep learning \cite{conn2009introduction},  such as online problem 
in bandit setting 
\cite{flaxman2005online,agarwal2010optimal,wibisono2012finite}, certain 
graphical model and structure-prediction problems \cite{spall2005introduction}, 
and black-box 
attack to deep neural networks (DNNs) 
\cite{bhagoji2017exploring,chen2017zoo,narodytska2017simple}. However, the 
convergence rate of current approaches encounters a factor of 
$d$, where $d$ is problem dimension. This prevents the application of 
derivative-free optimization in high-dimensional problems.

This paper focuses on  the theoretical 
development  of derivative-free (zeroth-order) method 
for non-convex optimization. More specifically, we consider the following 
optimization problem:
\begin{align}\label{SCSG-Zero:Definition:finite-sumf}
\mathop {\min }\limits_{x\in \mathbb{R}^d} f( x )\mathop  = 
\limits^{\mathrm{def}} 
\frac{1}{n}\sum\nolimits_{i = 1}^n 
{F( {x;{\xi _i}} )} ,
\end{align}
where $f(x )$ and $F( x,\xi _i ):{\mathbb{R}^d} 
\to \mathbb{R}$ are 
differentiable, non-convex 
functions, and  $\xi_i$, $i\in[n]$ is a random variable. In particular, when 
$n$=1, the objective function is $f(x )$=$F( x,\xi )$ with a fixed $\xi$, which 
becomes the problem 
solved in
\cite{nesterov2017random}.
{
	To solve~\eqref{SCSG-Zero:Definition:finite-sumf}, most 
	approaches~\cite{nesterov2017random}
	consider the use of stochastic zeroth-order oracle ($\mathcal{SZO}$). 
	At each iteration, for a given $x, u$ and $\xi$, 
	$\mathcal{SZO}$ outputs a stochastic gradient $G_{\mu}(x,u,\xi)$ defined by
	\begin{align}
	\label{SCSG-Zero:Definition:Gaussian:SmoothGradient}
	{G_\mu }\left( {x,\xi ,u} \right) =& \frac{{F\left( {x + \mu u,\xi } 
			\right) 
			- F\left( {x,\xi } \right)}}{\mu }u, 
	\end{align}
	which approximates the derivative along the direction of $u$.
	Each $\mathcal{SZO}$ only requires $2$ function value evaluations (or $1$ 
	if 
	$F(x, \xi)$ has already being queried). 
	It is thus natural to analyze the convergence rate of an algorithm in terms 
	of number of 
	$\mathcal{SZO}$ required to achieve
	$\|\nabla f(x)\|^2\leq \epsilon^2$ with a small $\epsilon$. 
}
%The 
%gradient of $f(x)$ is \cho{estimated} by a stochastic zeroth-order oracle 
%($\mathcal{SZO}$). More specifically, at each iteration, $x$, $u$ and  $\xi$ 
%being 
%the input, $\mathcal{SZO}$ output a stochastic gradient $G_{\mu}(x,u,\xi)$  
%defined 
%in (\ref{SCSG-Zero:Definition:Gaussian:SmoothGradient}), where $u$ is a 
%direction 
%vector. 
%In order to analyze 
%the zeroth-order method, we denote by  complexity 
%$T_{\mathcal{SZO}}$ the 
%number of computation of gradient  $F( x,\xi _i )$, $T$ the number of 
%iteration.
%Following the 
%classical benchmark of  non-convex optimization, we focus on our proposed 
%algorithm that can 
%efficiently find an approximate 
%stationary 
%point satisfying $\| \nabla f( x )\|^2 \le \varepsilon ^2$.

A recent important work by Nesterov and Spokoiny \cite{nesterov2017random} 
proposed the random gradient-free method (RGF) and proved some tight bounds 
for 
approximation the gradient through function value information with Gaussian 
smoothing techniques. He established an $O(d/\varepsilon^2)$   
complexity for non-convex smooth function in the case of $n$=1 in 
problem (\ref{SCSG-Zero:Definition:finite-sumf}). Subsequently,  Ghadimi and 
Lan \cite{ghadimi2013stochastic} introduced a  randomized 
stochastic gradient (RSG)  method for solving the stochastic programming 
problem~\eqref{SCSG-Zero:Definition:finite-sumf} and proved the 
complexity of $O( d/\varepsilon ^2+ \sqrt d /\varepsilon ^4 
)$. However, when the dimension $d$ is large, especially in deep learning, 
these derivative-free methods 
will suffer slow convergence. 

The dependency in $d$ is mainly due to the 
variance in sampling query direction $u$. 
Recently, a family of variance reduction methods have
been proposed for first-order optimization, including 
SVRG \cite{johnson2013accelerating},  
SCSG \cite{lei2017non} and Natasha \cite{allen2017natasha}.
They developed ways to reduce variance of stochastic samples ($\xi$). 
%They divide iterations into outer
%and inner loops, and use the outer-loop gradient estimator to reduce the 
%variance
%of inner loop gradient estimators. 
%evaluate gradient at the beginning of each outer-loop as 
%reference to reduce variance of inner-loop gradients.  
%Stochastic optimization for solving problem 
%(\ref{SCSG-Zero:Definition:finite-sumf}) with known functions' structure have 
%sprung  up many classical methods such as SVRG 
%\cite{johnson2013accelerating},  
%SCSG \cite{lei2017non} and Natasha \cite{allen2017natasha}. Their common point 
%is applying the technical of variance reduction to reduce the complexity in 
%large-scale data that dividing the iteration into epochs and computing the 
%full 
%gradient or the estimated gradient with less than $n$ as the 
%fixed 
%gradient for each epoch.  
%In this paper, we consider the following
It is thus natural to ask the following question: can the variance reduction 
technique also be used in
derivative-free optimization to 
reduce the
$\mathcal{SZO}$ complexity caused by problem dimension?  And how to choose the 
best size of Gaussian random vector set for each epoch to estimate the  
gradient 
in zeroth-order optimization?

In this paper, we develop a novel stochastic zeroth-order method with variance 
reduction under Gaussian smoothing (SZVR-G). The main contributions are 
summarized below. 
\begin{itemize}
	\item We proposed a novel algorithm based on variance reduction. Different 
	from RSG and RGF that generate a Gaussian random 
	vector for each iteration, we independently generate 
	Gaussian vector set 
	(in practice, we 
	preserve the corresponding seeds) to compute the average of  direction 
	derivatives at 
	the beginning of  each epoch as defined in 
	(\ref{SCSG-Zero:SCSG:Definition-SmoothGradient-Gaussian-mu-block}). In the 
	inner iteration of epoch, we 
	randomly 
	select one or block of seeds that preserved in the outer epoch to compute 
	the 
	corresponding gradient as defined in 
	(\ref{SCSG-Zero:SCSG:Definition-EstimateGradient-Gaussian}).
	\item We give the theoretical proof for the proposed algorithm and show 
	that our results are better than that of RGF and  RSG in both smooth and 
	non-smooth functions, and in the case of both $n=1$ and $n \ne 1$  of 
	problem (\ref{SCSG-Zero:Definition:finite-sumf}). Furthermore, we also 
	explicitly present 
	parameter settings and the corresponding derivation 
	process, which is better for understanding the convergence analysis.
	\item We extend the stochastic zeroth-order optimization to the 
	mini-batch setting.  Although the $\mathcal{SZO}$ complexity will increase, 
	we show that the increasing rate is sublinear to batch size. In comparison, 
	previous algorithms including RGF and RSG have complexity growing linearly 
	with batch 
	size.   Furthermore, the total number of iterations in our algorithm will 
	decrease when using larger 
	mini-batch, which implicitly implies better parallelizability.
	\item We show that our algorithm is more efficient than both RGF and RSG in 
	canonical logistic regression problem. Furthermore, we successfully apply 
	our algorithm
	to a real black-box adversarial attack problem that involves 
	high-dimensional zeroth order optimization. 
\end{itemize}

\subsection{Our results}

\begin{table}[t]
	\centering
	\caption{Comparison of $\mathcal{SZO}$ complexity for the non-convex 
		problem (B is the number of random set $\mathcal B \subseteq  [n]$, d the dimension of the 
		$x$, 
		and $b_0$ 
		is the number of mini-batch)}
	\label{SCSG-Zero:Table:SZOcomplexity}
	\begin{tabular}{c|l|l|c}
		\hline
		Method 
		& \multicolumn{1}{c|}{$\mathcal{SZO}$ complexity } 
		& \multicolumn{1}{c|}{Mini-Batch} 
		&\multicolumn{1}{c}{Non-smooth} \\
		\hline
		{RGF \cite{nesterov2017random} }  
		&  $O\left( {\frac{dn}{{{\varepsilon ^2}}}} \right)$      
		&  $O\left( {\frac{dn}{{{\varepsilon ^2}}}b_0} \right)$ 
		&$O\left( {\frac{{{d^5}n}}{{{\varepsilon ^5}}}} \right)$\\
		\hline
		RSG \cite{ghadimi2013stochastic}
		& $O\left( \frac{d}{\varepsilon ^2}+\frac{\sqrt{d}}{\varepsilon ^4} 
		\right)$      
		&   $O\left( \left( \frac{d}{\varepsilon 
			^2}+\frac{\sqrt{d}}{\varepsilon 
			^4}
		\right)b_0 \right)$ 
		&-\\
		\hline
		{SZVR-G}  
		& $O\left( {\max \left\{ 
			{\frac{{{d^{\frac{2}{3}}}{B^{\frac{1}{3}}}}}{{{\varepsilon 
							^2}}},\frac{{{d^{\frac{1}{3}}}{B^{\frac{2}{3}}}}}{{{\varepsilon
							^2}}}} 
			\right\}} \right)$  
		& $O\left( {\max \left\{ {\frac{{{{\left( {d{b_0}} 
								\right)}^{\frac{2}{3}}}{B^{\frac{1}{3}}}}}{{{\varepsilon
							
							^2}}},\frac{{{{\left( {d{b_0}} 
								\right)}^{\frac{1}{3}}}{B^{\frac{2}{3}}}}}{{{\varepsilon
							^2}}}} 
			\right\}} \right)$  
		&$O\left( {\frac{{{d^{\frac{5}{3}}}{B^{\frac{1}{3}}}}}{{{\varepsilon 
						^{\frac{{11}}{3}}}}}} \right)$\\
		\hline
	\end{tabular}%
	\label{tab:addlabel}%
\end{table}%

%\begin{wrapfigure}{R}{0.55\textwidth}
%	%\vspace{-20pt}
%	\begin{center}
%		\includegraphics[width=0.42\textwidth]{PlotComplexity-9}
%	\end{center}
%	%\vspace{-0pt}
%	\caption{The SZO complexities of RSG and SZVR-G with different dimension 
%		$d$. {Note that in the plot we assume $n$ is infinity (we set 
%			$B=1/\varepsilon^2$). } }
%	\vspace{-20pt}
%	\label{SCSG:Figure:Complexity}
%\end{wrapfigure}

Our proposed algorithm can achieve the following  $T_{\mathcal{SZO}}$ 
complexity:  
%In this paper, our proposed we obtain results $T_{\mathcal{SZO}}$ of our 
%proposed method   
%is 
\begin{align*}
O\left( {\frac{1}{{{\varepsilon 
				^2}}}\max \{ 
	{d^{2/3}}{B^{1/3}},{d^{1/3}}{B^{2/3}}\} } \right),
\end{align*}
where  $B=\min\{n,1/\varepsilon^2\}$. We identify an interesting dichotomy with 
respect to $d$. In particular, if 
$d\ge B$, $T_{\mathcal{SZO}}$ becomes  $O(d^{2/3}B^{1/3}/{\varepsilon^2})$, 
otherwise $T_{\mathcal{SZO}}$ becomes $O(d^{1/3}B^{2/3}/{\varepsilon^2})$. 
Different complexities of methods are presented in 
Table~\ref{SCSG-Zero:Table:SZOcomplexity}. 
%Figure 
%\ref{SCSG:Figure:Complexity} clearly shows the phenomenon.

Comparing our method with RGF \cite{nesterov2017random} in the case of $n$=1 
(that is $B=1$), 
%in Table \ref{SCSG-Zero:Table:SZOcomplexity}, 
we can see that our result is 
better than 
that of RGF  with a factor of $d^{1/3}$ improvement.  For $n>1$, the complexity 
of  our method is 
also better 
than that of RSG \cite{ghadimi2013stochastic} as clearly shown in Figure 
\ref{SCSG:Figure:Complexity}.

\begin{figure}[t]
	\begin{center}
		\includegraphics[width=0.5\textwidth]{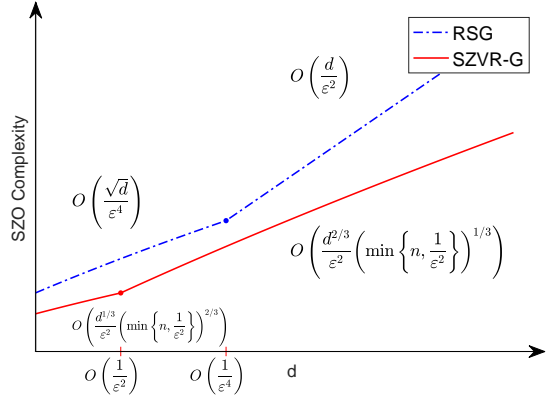}
	\end{center}
	\caption{The SZO complexities of RSG and SZVR-G with different dimension 
		$d$. {Note that in the plot we assume $n$ is infinity (we set 
			$B=1/\varepsilon^2$). } }
	\label{SCSG:Figure:Complexity}
\end{figure}

\textbf{{Mini-Batch}}
Our result generalizes to the mini-batch stochastic setting, where in  the 
inner iteration of each epoch, the estimated gradient ${{\tilde \nabla }_k} $ 
defined in 
(\ref{SCSG-Zero:SCSG:Definition-EstimateGradient-Gaussian}) is computed 
{with mini-batch of $b_0$ times. }
%repeatedly  for $b_0$ times and then average them. 
The $\mathcal{SZO}$ 
complexity will become
%need to be adjusted by multiplying a factor $b_0^{1/3}$ or $b_0^{2/3}$, that 
%is  
$O(\frac{1}{\varepsilon 
	^2}\max \{ 
d^{2/3}B^{1/3}b_0^{2/3},d^{1/3}B^{2/3}b_0^{1/3} \})$. The comparison of  
mini-batch $\mathcal{SZO}$ complexity is also shown in Table 
\ref{SCSG-Zero:Table:SZOcomplexity}.

\textbf{{Non-smooth}} We also give the convergence analysis for non-smooth 
case and present the $\mathcal{SZO}$ complexity, which is better than that 
of RGF \cite{nesterov2017random}.

\subsection{Other Related work}
Derivative-free optimization can be dated back to the early days of the 
development of the optimization theory \cite{matyas1965random}. The advantage 
of using derivative-free method is manifested in the case when computation of 
function value is much simpler than gradient, or in the 
black-box 
setting when optimizer does not have full information about the function. 
% With the renaissance 
%of 
%artificial intelligence, deep 
%learning 
%research has 
%given rise to the numerous architecture for neural networks and achieved the 
%state-of-the-art performance in many machine learning task that were once to 
%be 
%believed challenged. However, DNNs have identified to be vulnerable to the 
%adversarial example. Thus, the study on the black-box  for generating 
%adversarial example becomes an important 
%field \cite{bhagoji2017exploring,chen2017zoo,narodytska2017simple}.

The most common method for derivative-free optimization is the random 
optimization approach 
\cite{matyas1965random}, 
which samples a random vector uniformly distributed over the unit sphere,  
computes the directional derivative of the function, and then moves 
the next point if the update leads to the decrease of function value. 
However, no particular convergence rate was established. Nesterov and 
Spokoiny \cite{narodytska2017simple} presented several random derivative-free 
methods, and 
provide the corresponding complexity bound for both convex and non-convex 
problems. What's more, an important kind of smoothness, Gaussian smoothing and 
its properties were established. Ghadimi and Lan \cite{ghadimi2013stochastic}  
incorporated the Gaussian smoothing technique to randomized stochastic gradient 
(RSG). John \textit{et al}.~\cite{wibisono2012finite,duchi2015optimal} analyzed 
the 
finite-sample convergence rate of zeroth-order optimization for  convex 
problem. 	
Wang \textit{et al}. \cite{wang2017stochastic} considered the zeroth-order 
optimization in 
high-dimension and  Liu \textit{et al}
\cite{liu2017zeroth} present ADMM method for zeroth-order optimization, but also in convex 
function. For the coordinate 
smoothness (the sampled direction is along natural basis), 
Lian \textit{et al}.~\cite{lian2016comprehensive} presented zeroth-order under 
asynchronous 
stochastic parallel optimization for non-convex problem. Subsequently, Gu 
\textit{et 
	al}.~\cite{gu2016zeroth} 
apply variance reduction of zeroth-order to 
asynchronous doubly stochastic algorithm, however, without the specific 
analysis of the complexity related to dimension $d$. Furthermore, it is not 
practical to perform full gradient computation in the parallel setting for 
large-scale data.

Stochastic first-order 
methods including SGD~\cite{nesterov2013introductory} and 
SVRG~\cite{reddi2016stochastic}
have been studied extensively. 
%for non-convex problem have been studied extensively. there emerged many 
%stochastic gradient(first-order) method for non-convex problem,  
%such as SGD 
%\cite{ nesterov2013introductory} and  
%SVRG \cite{reddi2016stochastic}. 
However, these two algorithms suffer 
from either hight iteration complexity or the complexity that depend on the 
number of samples. Lei \textit{et al.} \cite{lei2017non} recently proposed the 
stochastically 
controlled stochastic gradient (SCSG) method to obtain the complexity that is 
based on $\min\{n,1/\varepsilon^2\}$, which is derived from 
\cite{harikandeh2015stopwasting} and \cite{lei2017less} for the convex case.

The rest of the paper is organized as following. We first introduce some notations, 
definitions and assumptions in Section \ref{SCSG-Zero: 
	Section:Preliminary}. 
  %we introduce some notations, definitions and 
	%assumption. 
  In 
Section \ref{SCSG-Zero: Section:SCVR-G}, we provide our algorithm via variance 
reduction technology, and analyze the complexity for both smooth and non-smooth 
function, and their corresponding mini-batch version. Experiment results are 
shown in
\ref{SCSG-Zero: Section:Exp}. Section
\ref{SCSG-Zero: Section:con} concludes our paper.
\section{Preliminary}\label{SCSG-Zero: Section:Preliminary}
Throughout this paper, we use Euclidean norm denoted by $\|\cdot\|$. We use $i 
\in [ n ]$  to denote that $i$ is generated from $[n] = \{ {1,2,...,n} \}$. We 
denote by ${{{\cal B}}}$ and ${{{\cal D}}}$ the set, and  $B = 
\left| {\cal B}\right|$ and $D = \left| {\cal D}\right|
$ the cardinality of the sets. We 
use ${{\xi _{\cal B}}}$ and  ${{u _{\cal D}}}$ to denote the variable set, 
where  ${{\xi _{{\cal B}\left[ i \right]}}}$ belong 
to  ${{\xi _{\cal B}}}$, $i\in \cal B$, and ${{u _{{\cal D}\left[ j \right]}}}$ 
belong to ${{u _{\cal D}}}$,  $j\in \cal D$. We use $\mathbb{I}[event]$ to 
denote the indicator function of a probabilistic event.
Here are some definitions on the smoothness of a function, direction derivative 
and 
smooth approximation function and its property. 
\begin{Def}
	For a function $f(x)$ : $\mathbb{R}^d \to \mathbb{R}$, $\forall x,y 
	\in{\mathbb R^d}$, 
	\begin{itemize}
		\item $f(x) \in C^{0,0}$, then 
		$\left| 
		{f\left( x \right) - f\left( y 
			\right)} 
		\right| \le {L_0}\left\| {x - y} \right\|$.
		\item $f(x) \in C^{1,1}$, then  
		$\| \nabla f( x ) - \nabla f( y ) \| \le L_1\|x - y \|$ and 	
		$f( y ) \le f( x  ) + \langle \nabla f( x,\xi  ),y - x \rangle  
		+ \frac{L_1}{2}\| y - x \|^2$.	
	\end{itemize}

\end{Def}		
Note that if $F(x,\xi) \in C^{1,1}$, then $f(x)\in C^{1,1}$ due to  the fact 
that
\begin{align*}
\| {\nabla f( x ) - \nabla f( y )} \| =& 
\| {{\mathbb{E}_\xi }[ {\nabla F( {x,\xi } )} ] 
	- {\mathbb{E}_\xi 
	}[ {\nabla F( {y,\xi } )} ]} \|
\le {\mathbb{E}_\xi }[ {\| {\nabla F( {x,\xi } ) - 
		\nabla 
		F( {y,\xi } )} \|} ] \le {L_1}\| {x - y} 
\|.
\end{align*}
%\begin{Def} For a function ${F\left( {x,\xi } \right)}$, a vector $u$ 
%	and $\mu>0$,
%	\begin{align}
%		\label{SCSG-Zero:Definition:Gaussian:SmoothGradient}
%		{G_\mu }\left( {x,\xi ,u} \right) =& \frac{{F\left( {x + \mu u,\xi } 
%				\right) 
%				- F\left( {x,\xi } \right)}}{\mu }u.
%	\end{align}
%\end{Def}	
%Take expectation with respect to $\xi$, we have
%$
%{G_\mu }\left( {x,u} \right) = {\mathbb{E}_\xi }\left[ {{G_\mu }\left( {x,\xi 
%,u} 
%	\right)} \right] = \frac{1}{n}\sum\nolimits_{i = 1}^n {{G_\mu }\left( 
%	{x,u,{\xi 
%			_i}} \right)} 
%$. 
%\begin{align}
%\nabla {f_\mu }\left( x \right) = {E_u}\left[ {{G_\mu }\left( {x,u} \right)} 
%\right] = {E_{\xi ,u}}\left[ {{G_\mu }\left( {x,\xi ,u} \right)} \right].
%\end{align}
\begin{Def}\label{SCSG-Zero:Gaussian:Definition-fmu-Def}
	The smooth approximation of $f(x)$ is defined as 
	\begin{align}\label{SCSG-Zero:Gaussian:Definition-fmu}
	{f_\mu }( x ) = \frac{1}{{( {2\pi } )^{n/2}}}\int 
	{f( x + \mu u ){e^{\frac{1}{2}{\| u \|^2}}}du}, 
	\mu >0.
	\end{align}
\end{Def}

Its corresponding gradient is $\nabla {f_\mu }( x ) =  {\mathbb{E}_{\xi 
		,u}}[ {{G_\mu }( {x,u,\xi } )} ]$ and $\nabla {F_\mu }(x,\xi ) = 
{\mathbb{E}_u}[{G_\mu }(x,u,\xi )]$, where 
${{G_\mu }( x,u,\xi  )}$ defined in 
(\ref{SCSG-Zero:Definition:Gaussian:SmoothGradient}). The details of gradient 
derivation process can be referred to \cite{nesterov2017random}.

\begin{lemma} \label{SCSG-Zero:Gaussian:Lemma:Property-fmu}
	\cite{nesterov2017random} For $f_{\mu}(x)$ defined in 
	(\ref{SCSG-Zero:Gaussian:Definition-fmu}), 
	\begin{itemize}
		\item If $f \in {C^{0,0}}$, then 
		${f_\mu } \in {C^{1,1}}$ 	with  $L_1\left( {{f_\mu }} \right) = 
		\frac{{{1}}}{\mu }d^{1/2}{L_0}$, and  $\left| {{f_\mu }\left( x 
			\right) - f\left( x \right)} \right| 
		\le \mu {L_0}{d^{1/2}}$.	
		\item If $f \in C^{1,1}$, then 	${f_\mu } 
		\in C^{1,1},{L_1}( {{f_\mu 
		}} ) \le {L_1}( f )$, and 
		\begin{align*}
		{\| {\nabla {f }( {{x}} )} \|^2} \le& {\r1 
			2}{\| 
			{\nabla {f_\mu }( {{x}} )} \|^2} + \frac{1}{2}{\mu 
			^2}{L_1^2}{( {d+6} )^3},\\
		{\mathbb{E}_u}{\| {{G_\mu }( {x,u,\xi } ) - \nabla 
				{f_\mu }( x )} \|^2} \le& \frac{\mu^2 
		}{2}L_1^2{( {d 
				+ 6} )^3} + 2( {d + 4} )\| {\nabla 
			f( x 
			)} \|^2.
		\end{align*}
	\end{itemize}
\end{lemma}

\begin{assumption}\label{SCSG-Zero:Assumption-smoothfunctinoVarianceBound} We 
	assume that $H$ is the upper 
	bound on the variance of function ${\nabla {F_\mu }\left( {x,{\xi _i}} 
		\right)}$, that 
	is
	\begin{align*}
	\frac{1}{n}\sum\nolimits_{i = 1}^n {{{\left\| {\nabla {f_\mu 
					}\left( 
					x \right) - \nabla {F_\mu }\left( {x,{\xi _i}} \right)} 
				\right\|}^2}}  
	\le H.
	\end{align*}
\end{assumption}

\section{Stochastic Zeroth-order via Variance reduction with Gaussian 
	smooth}\label{SCSG-Zero: Section:SCVR-G}
We introduce  our SZVR-G method in 
Algorithm \ref{SCSG-Zero:SCSG:Algorithm:Gaussian:VR-DB}. At each outer 
iteration, we have two kinds of sampling: the first one is to form 
$\xi_{\cal B}$ with the size of $B$, which are randomly selected from $[n]$; 
the 
second one is to 
independently generate a  Gaussian vector set $u_{\cal D}$ with $D$ times. 
Furthermore, we store the corresponding seeds of Gaussian vectors, which will 
be used  for the inner iterations. The 
main difference between set $\cal D$ and $\cal B$ is the property of 
independence, which will be the key element in analyzing the size of their 
sets. Based on these two sets, we compute the random gradient at a snapshot 
point ${{\tilde 
		x}_s}$, which is maintained for each epoch,
\begin{align}\label{SCSG-Zero:SCSG:Definition-SmoothGradient-Gaussian-mu-block}
{G_\mu }( {{{\tilde 
			x}_s},{u_\mathcal{D}},{\xi _{\cal B}}} ) = 
\frac{1}{D}\sum\nolimits_{j = 
	1}^D {\frac{1}{B}\sum\nolimits_{i = 1}^B {{G_\mu }( 
		{{{\tilde 
					x}_s},{u_{\mathcal{D}[ j 
					]}},{\xi _{\mathcal{B}[i]}}} )} }, 
\end{align}	
where the 
definition of ${G_\mu 
}( x,u,\xi )$ is in 
(\ref{SCSG-Zero:Definition:Gaussian:SmoothGradient}).

At each inner iteration, we select  $i$ and $j$ from $[n]$ and $\cal D$ 
randomly, and compute the estimated random gradient, 
\begin{align}\label{SCSG-Zero:SCSG:Definition-EstimateGradient-Gaussian}
{{\tilde \nabla }_k} = {G_\mu }( {{x_k},{u_{{\cal D}[ j 
			]}},{\xi _{ i }}} ) - {G_\mu }( 
{{\tilde 
		x}_s},{u_{{\cal 
			D}[ j ]}},\xi _{ i } ) + {G_\mu }( {{{\tilde 
			x}_s},u_{\cal D},\xi_{\mathcal{B}} 
} ),
\end{align}
where $u_{\cal D}$  and $\xi_{\cal B}$ are the Gaussian vector set and sample 
set.
Taking  expectation of ${\tilde \nabla }_k$ with respect to $i$, $j$ and $u$, 
we 
have
\begin{align}\label{SCSG-Zero:SCSG:Definition-EstimateGradient-Gaussian-Expect}
{\mathbb{E}_{i,j,u}}[ {{{\tilde \nabla }_k}} ] = \nabla {f_\mu }\left( 
{{x_k}} 
\right)- \nabla {f_\mu }( {{\tilde x_s}} ) + 
\nabla 
{F_\mu }({\tilde x}_s,\xi _\mathcal{B} ), 
\end{align}
where $\nabla {f_\mu }(x)$ and $\nabla {F_\mu }(x,\xi )$ are defined in 
Definition 
\ref{SCSG-Zero:Gaussian:Definition-fmu-Def}.
\begin{algorithm}[t]
	\caption{Zeroth-order via variance reduction with Gaussian smooth}
	\label{SCSG-Zero:SCSG:Algorithm:Gaussian:VR-DB}
	\begin{algorithmic}
		\Require $K$, $S$, $\eta$ (learning rate), and $\tilde{x}_0$
		%		\ENSURE Initialize $\tilde{\theta_1}$.
		\For{$s =0,1, 2,\cdots,S-1$}
		\State \textbf{Independently} Generate  Gaussian vector 
		set $u_{\cal D}$ through 
		Gaussian random	vector generator  with 
		$D$ times, where  ${\cal D}$ is the index set.  \Comment{In practice, 
			store  Gaussian random vector seeds for each $s$th iteration.}
		\State Sample from $[n]$ to form mini-batch $\mathcal{B}$ with 
		$|\mathcal{B}|=B$.
		\State $x_0=\tilde x_s$
		\State $G = {G_\mu }( {x_0,{u_{\cal D}},\xi_{\mathcal{B}} } )$ 
		\Comment{(\ref{SCSG-Zero:SCSG:Definition-SmoothGradient-Gaussian-mu-block})}
		\For{$k =0,1,2,\cdots,K-1$}
		\State
		Sample ${i}$ from $[n]$ and $ {j}$ from $\cal D$ 
		\State
		${{\tilde \nabla }_k} = {G_\mu }\left( {{x_k},{u_{{\cal D}\left[ j 
					\right]}},{\xi _{ i }}} \right) - {G_\mu }\left( 
		{{{\tilde 
					x}_s},{u_{{\cal 
						D}\left[ j \right]}},{\xi _{i }}} \right) + G$ 
		\State
		${x_{k + 1}} = {x_k} - \eta {{\tilde \nabla }_k}$%\Comment{4 Queries}
		\EndFor
		\State Update $\tilde{x}_{s+1}=x_K $
		\EndFor \\	
		\textbf{Output:}  $ x_k^s$, $s \in \left\{ {1,...,S} \right\},k 
		\in \left\{ {1,...,K} \right\}$ 
	\end{algorithmic}
\end{algorithm}
\subsection{Convergence analysis}
We present the convergence and complexity results for our algorithm. Theorem 
\ref{SCSG-Zero:SCSG:Theorem:new:SZO-Convergence} is based on the variance 
reduction technique for the non-convex problem. The detailed proof can 
be found in Appendix 
\ref{NS-SCSG-Zero:Appendix:smooth}. In order to ensure the convergence, we 
present the parameter settings, such as $c_0$, $q$, $K$, $w_0$ and $D$ in
Remark \ref{SCSG-Zero:Remark:parameter-u0} and 	
\ref{SCSG-Zero:Remark:parameter-C}. %{\color{red}what is $c_0$ and $q$ in our algorithm? }
%which are logical and reasonable.
\begin{theorem}\label{SCSG-Zero:SCSG:Theorem:new:SZO-Convergence}
	In Algorithm 
	\ref{SCSG-Zero:SCSG:Algorithm:Gaussian:VR-DB},  under Assumption 
	\ref{SCSG-Zero:Assumption-smoothfunctinoVarianceBound},
	for $F(x,\xi)$ $ \in $ $
	C^{1,1}$, let parameters $\mu, \eta, q,K>0$,  $c_0\le L_1$ and 
	the 
	cardinality of Gaussian vector set and sample set $D \ge O\left( {\eta d} 
	\right)$,  we have
	\begin{align*}
	\frac{1}{S}\sum\nolimits_{s = 0}^{S-1} {\frac{1}{K}\sum\nolimits_{k = 0}^{K - 
			1} 
		{{{\left\| {\nabla f\left( {x_k^s} \right)} \right\|}^2}} }  \le 
	\frac{{32R}}{{SK\eta }} + \frac{32}{\eta }{J_0}+ \frac{1}{2}{\mu 
		^2}L_1^2{(d 
		+ 6)^3},
	\end{align*}
	where   $x_*$ is the 
	optimal value of function $f_{\mu}(x)$, $R = {\max _x}\{ f_{\mu}( x ) - 
	f{_\mu}( x_* 
	):f{_\mu}( x ) \le f{_\mu}( x_0 ) \}$, and 
	\begin{align}
	%\label{SCSG-Zero:Theorem:SZO-Convergence-Jk}
	{J_{0}} =\frac{3}{4}\left( {\frac{1}{D} + \frac{{\mathbb{I}(B < n)}}{B} 
		+ 3} 
	\right)\left( {{L_1} + 2{c_{0}}} \right)\mu L_1^2{\left( {d + 6} 
		\right)^3}{\eta ^2} + \left( {1 + \frac{1}{q}{c_{0}}} 
	\right)\frac{1}{2}\eta \frac{{\mathbb{I}(B < n)}}{B}H.
	\end{align}
\end{theorem}
The 
$\mathcal{SZO}$ complexity is presented in Theorem 
\ref{SCSG-Zero:SCSG:theorem:new:SZO-Complexity}, which is based on the best 
choice of 
step size $\eta$. For the different sizes of $B$ and $d$, we give different 
results, which is an interesting phenomenon caused by two types of samples. 
\begin{theorem}\label{SCSG-Zero:SCSG:theorem:new:SZO-Complexity} In Algorithm 
	\ref{SCSG-Zero:SCSG:Algorithm:Gaussian:VR-DB},  for $F(x,\xi) \in 
	C_{}^{1,1}$,  
	let the size of  sample set $\cal B$, $B$ =$O (\min \{ 
	n,1/{\varepsilon 
		^2} 
	\})$,  the step 
	$\eta  =O (\min \{ 1/(d^{2/3}B^{1/3}),1/(d^{1/3}B^{2/3}) \})$,  $\mu  \le
	O(\varepsilon /({L_1} d 
	^{1.5}))$, and the 
	number of 
	inner iteration $K \le O (1/\max \{ d\eta ^2,d^{0.5}\eta 
	^{1.5} \})$,   Gaussian vectors set  
	$D \ge O(\eta d)$.
	In 
	order to 
	obtain 
	\begin{align*}
	\frac{1}{S}\sum\nolimits_{s = 0}^{S - 1} {\frac{1}{K}\sum\nolimits_{k = 
			0}^{K - 1} {{{\left\| {\nabla f\left( {{x_k^s}} \right)} 
					\right\|}^2}} }  \le 
	{\varepsilon ^2},
	\end{align*}
	the  total number of $T_{\mathcal{SZO}}$ is at most 
	$O(\frac{1}{{{\varepsilon 
				^2}}}\max \left\{ {{d^{2/3}}{B^{1/3}},{d^{1/3}}{B^{2/3}}} 
        \right\})$, with the number of total iterations  $T=SK>O(1/(\varepsilon^2 
        \eta))$.
\end{theorem}

\subsubsection{Variance reduction for Gaussian random direction}
If we only consider the  directions of Gaussian random vector, that is $n=1$, 
Algorithm 
\ref{SCSG-Zero:SCSG:Algorithm:Gaussian:VR-DB} is similar to SVRG but the 
variance reduction will be on random directions instead of random samples. 
In outer iteration, we independently produce Gaussian random vectors and 
compute the smoothed  gradient estimator ${G_\mu }( {x,{u_{\cal 
D}},\xi_{\mathcal{B}} } )$  in 
\eqref{SCSG-Zero:SCSG:Definition-SmoothGradient-Gaussian-mu-block}
(Here, we use $\xi$ to indicate the only sample).
Then in inner iteration, we 
randomly select a Gaussian vector, and compute the estimated gradient  as 
${{\tilde \nabla }_k}$ in 
\eqref{SCSG-Zero:SCSG:Definition-EstimateGradient-Gaussian}. 
Since this is the same problem solved in Nesterov and Spokoiny 
\cite{nesterov2017random}, we compare the $\mathcal{SZO}$ complexity 
between our method and theirs based on different step-size choices: 
%and w
%However, comparing 
%with Nesterov and Spokoiny \cite{nesterov2017random}, under different 
%dimension and
%step $\eta$, how about  the  $\mathcal{SZO}$ complexity? Based on the proof 
%result in 
%Theorem \ref{SCSG-Zero:SCSG:Theorem:new:SZO-Convergence}, we have that 

\begin{itemize}
  \item For $\eta >1/d$, we set $\eta=1/d^{2/3}$, the $\mathcal{SZO}$ complexity of 
	our 
	proposed method is ${T_{\mathcal{SZO}}}$ 
	= $O({d^{2/3}}/{\varepsilon ^{2}})$ which is better than that of 
	RGF
	\cite{nesterov2017random}, $O(d/{\varepsilon ^2})$. The corresponding 
	Gaussian 
	random vector set $\cal D$, $D>O(\eta d)>1$. 
	This is due to the fact that 
	${\nabla {f_\mu }( {\tilde x}_s )}$ is not finite-sum structure and 
	the term $\mathbb{E}_u\| G_\mu ( {\tilde x}_s,{u_D},{\xi 
		_{\cal B}}) - \nabla f_\mu ( {\tilde x}_s ) 
	\|^2 \ne 0$, which is bounded by $O( {{\mu ^2}{d^3} + d{{\| {\nabla 
					f( x )} \|}^2}} )/D$. More details can be referred to Lemma 
	\ref{SCSG-Zero:SCSG:Lemma:new:Bound-E[x-Ex]-Onevariance}. 
	%{This sentence  is not clear} 
	This is the key difference with SVRG method 
	\cite{reddi2016stochastic}. Based on the lower bound, we can derive the 
	corresponding best complexity and best step as shown in Theorem 
	\ref{SCSG-Zero:SCSG:theorem:new:SZO-Complexity}.
	
	\item For $ \eta \le 1/d$, the $\mathcal{SZO}$ complexity will be larger 
	than $O(d/{\varepsilon ^2})$.	This can be directly seen from 
	the total number of $T$. In this case $D$ becomes 1, and the proposed 
	algorithm 
	will become the original  RGF 
	\cite{nesterov2017random} method, where the step is  $O(1/d)$. This 
	can explain that why the variance reduction method is better than that of 
	RGF, that is our proposed method can apply the large  step to obtain the 
	better 
	complexity. 
	%\cho{discuss}
\end{itemize}

\subsubsection{Variance reduction for  finite-sum function}
For the finite-sum function as in (\ref{SCSG-Zero:Definition:finite-sumf}), In 
Algorithm \ref{SCSG-Zero:SCSG:Algorithm:Gaussian:VR-DB}, we also provide the 
variance-reduction technique at the same time for both Gaussian vector and 
random variable $\xi$. Our algorithm 
has two kinds of random procedure. 
That is, in outer iteration, we compute the gradient  include both $B$ samples 
and $D$ Gaussian random vectors. 
In inner iteration, we randomly select a sample and a Gaussian random vector to 
estimate the gradient. 
Here, we compare our result with RSG 
\cite{ghadimi2013stochastic}, which also use both random sample and Gaussian 
random vector. Based on the result in  Theorem
\ref{SCSG-Zero:SCSG:theorem:new:SZO-Complexity},
we 
discuss the $\mathcal{SZO}$ 
complexity under different $d$,
\begin{itemize}
	\item For $d < B$, the $\mathcal{SZO}$ complexity of 
	our 
	proposed method is ${T_{\mathcal{SZO}}}$	=$O( {d^{1/3}}{B^{2/3}}/{\varepsilon ^2})$. 
	This result is similar to SCSG \cite{lei2017non} if the dimension d is not 
	large enough. Furthermore, in our algorithm, we set B as the fix value 
	rather than a value that is produced by the probability. If $B=d$, the 
	complexity result looks the same as RSG \cite{ghadimi2013stochastic}. But 
	the difference lies on that the $B$ is no more than $1/\varepsilon^2$ such 
	that our result is better than  RSG \cite{ghadimi2013stochastic}. 
	Figure 
	\ref{SCSG:Figure:Complexity} clearly shows the difference.
	
	\item For $d > B$, the $\mathcal{SZO}$ complexity  becomes
	$O( {d^{2/3}}{B^{1/3}}/{\varepsilon ^2})$. 
	%We can see that $d$ plays more 
	%influence than $B$ for the complexity. Furthermore, 
	The complexity is also 
	better than that of RSG. 
\end{itemize}

Based on above discussions, we conclude that the $\mathcal{SZO}$ complexity of 
our proposed method is better than that of RSG 
\cite{ghadimi2013stochastic} and RGF \cite{nesterov2017random}.
\subsection{Mini-batch SZVR-G}
We extend the SZVR-G to the mini-batch version in Algorithm 
\ref{SCSG-Zero:SCSG:Algorithm:Gaussian:VR-DB-Block}, which is similar to 
Algorithm 
\ref{SCSG-Zero:SCSG:Algorithm:Gaussian:VR-DB}. The difference is that we 
estimate the gradient in inner epoch with $b_0$ times computation, then average 
them. Theorem \ref{SCSG-Zero:SCSG:theorem:new:SZO-Complexity-Block} gives the 
corresponding complexity and the corresponding step size.

\begin{theorem}\label{SCSG-Zero:SCSG:theorem:new:SZO-Complexity-Block} In 
	Algorithm \ref{SCSG-Zero:SCSG:Algorithm:Gaussian:VR-DB-Block}, under 
	Assumption 
	\ref{SCSG-Zero:Assumption-smoothfunctinoVarianceBound}, for $F(x,\xi) 
	\in 
	C^{1,1}$,  
	let the size of  the sample set $\cal B$, $B$ =$O (\min \{ 
	n,1/{\varepsilon 
		^2} 
	\})$, the step 
	$\eta$  =$ O(\min \{ b_0^{1/3}/(d^{2/3}B^{1/3}), b_0^{2/3}/(d^{1/3}B^{2/3}) 
	\})$,  $\mu  \le
	O(\varepsilon /({L_1} d 
	^{1.5}))$, and the 
	number of 
	inner iteration $K \le O (1/\max \{ d\eta ^2,d^{0.5}\eta 
	^{1.5} \})$,   Gaussian vectors set  
	$D \ge O(\eta d)$.
	In 
	order to 
	obtain 
	\begin{align*}
	\frac{1}{S}\sum\nolimits_{s = 0}^{S - 1} {\frac{1}{K}\sum\nolimits_{k = 
			0}^{K - 1} {{{\left\| {\nabla f\left( {{x_k^s}} \right)} 
					\right\|}^2}} }  \le 
	{\varepsilon ^2},
	\end{align*}
	the  total number of $T_{\mathcal{SZO}}$ is at most $O(\max \{ 
	{{d^{2/3}}{B^{1/3}}b_0^{2/3},{d^{1/3}}{B^{2/3}}b_0^{1/3}} \})
	$, with number of total iterations   $T>O(1/(\varepsilon^2 
	\eta))$.	
\end{theorem}
From the above Theorem, we can see that the $\mathcal{SZO}$ complexity is 
increased by a factor $b_0^{2/3}$ or $b_0^{1/3}$, which is smaller than the 
size 
of the mini-batch. However, the corresponding complexity of RGF and RSG
will be increased by multiplying a factor of $b_0$ (see Table~\ref{tab:addlabel}), so our algorithm has a 
better dependency to the batch size. Furthermore, our total 
number of iterations will decrease by a factor $b_0^{2/3}$ or $b_0^{1/3}$.

\subsection{SZVR-G for non-smooth function}
For non-smooth function, we also provide the theory analysis and give the 
corresponding $\mathcal{SZO}$ complexity.
Similar to Theorem \ref{SCSG-Zero:SCSG:theorem:new:SZO-Complexity},  we analyze 
the convergence based on the norm of the gradient. But the 
difference 
lies in that the convergence of gradient norm  is ${\| {\nabla {f_\mu }( 
		x )} \|^2}$ rather than ${\| {\nabla {f }( x )} 
	\|^2}$. As stated in~\cite{nesterov2017random}, allowing $\eta  \to 0$ and 
	$\mu  
\to 
0$, the convergence of ${\| {\nabla {f_\mu }( 
		x )} \|^2}
$ ensures the convergence to a stationary point of the initial function. 
%Following the analysis in Nesterov and Spokoiny \cite{nesterov2017random}, 
%\textit{``as $\eta  \to 0$ and $\mu  \to 0$, we can ensure convergence of the 
%	scheme to a stationary point of the initial function $f$. Due to the long 
%	proof 
%	and technical, we omit it".}

\begin{theorem}\label{NS-SCSG-Zero:SCSG:theorem:new:SZO-Complexity} In 
	Algorithm 
	\ref{SCSG-Zero:SCSG:Algorithm:Gaussian:VR-DB}, for $F(x,\xi) \in 
	C^{0,0}$,    the step 
	$\eta  = O({\varepsilon ^{5/3}}/( d^{5/3}B^{1/3} ))$,  $\mu  
	\le
	O(\varepsilon /({L_0} d 
	^{1/2}))$, and the 
	number of 
	inner iteration $K \le O(\varepsilon^2/(d^2\eta^2))$,   Gaussian vectors 
	set  
	$D \ge O(\eta d^3/\varepsilon^3)$.
	In 
	order to 
	obtain 
	\begin{align*}
	\frac{1}{S}\sum\nolimits_{s = 0}^{S - 1} {\frac{1}{K}\sum\nolimits_{k = 
			0}^{K - 1} {{{\left\| {\nabla f_{\mu}\left( {{x_k^s}} \right)} 
					\right\|}^2}} }  \le 
	{\varepsilon ^2},
	\end{align*}
	the  total number of $T_{\mathcal{SZO}}$ is 
	$O({{{d^{5/3}}{B^{1/3}}}}/{{{\varepsilon ^{11/3}}}})$, {number of inner 
		iterations}  
	$T>O(1/(\varepsilon^2 
	\eta))$.	
\end{theorem}
For the case of $n=1$, we can see that our $\mathcal{SZO}$ complexity is better 
than the $O(d^3/\varepsilon^5)$ complexity of RGF~\cite{nesterov2017random}. 
RSG 
\cite{ghadimi2013stochastic} do not provide the complexity of non-smooth 
function. Additionally, the mini-batch version for the non-smooth function  is 
similar to Theorem 
\ref{SCSG-Zero:SCSG:theorem:new:SZO-Complexity-Block}. We present the results 
in Theorem \ref{NS-SCSG-Zero:SCSG:theorem:new:SZO-Complexity-Block}.
%which can refer more 
%details information. 

\section{Experimental results}\label{SCSG-Zero: Section:Exp}
\subsection{Logistic regression with stochastic zeroth-order method}

\begin{figure}
	\centering
	\subfigure{
		\begin{minipage}[b]{0.35\textwidth}
			\includegraphics[width=1.0\textwidth]{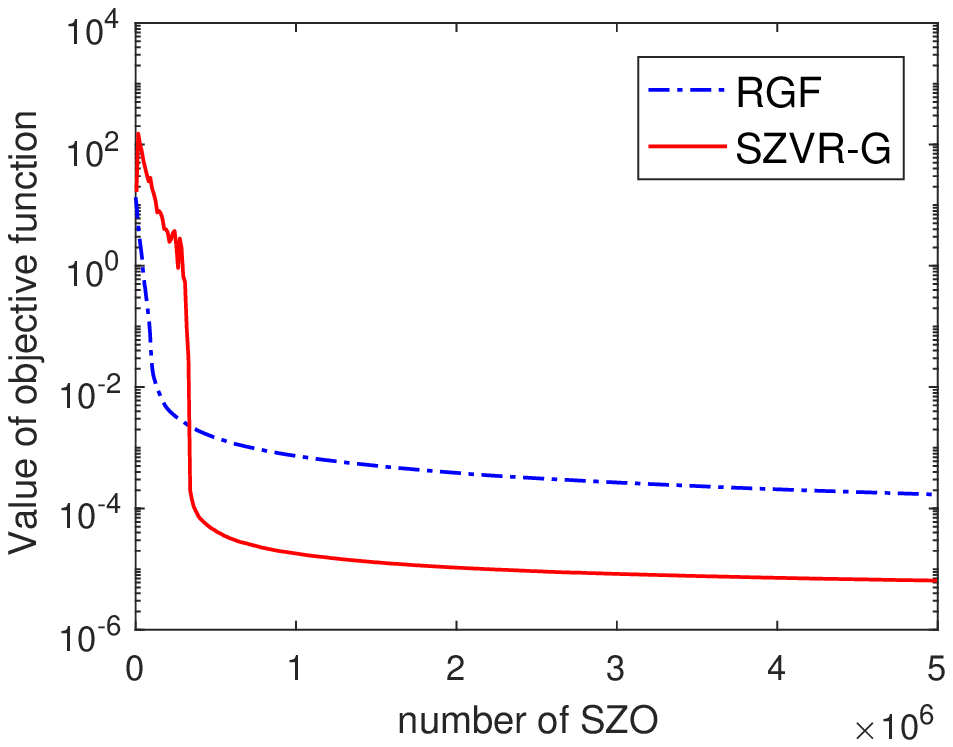}\\
			\includegraphics[width=1.0\textwidth]{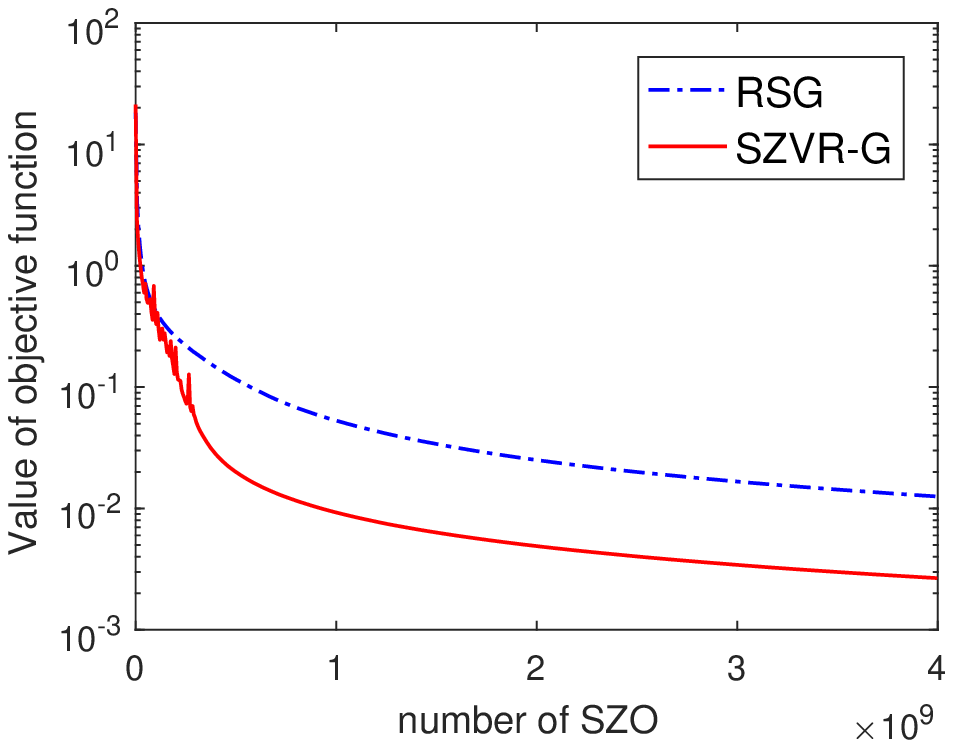}
		\end{minipage}
	}
	\subfigure{
		\begin{minipage}[b]{0.35\textwidth}
			\includegraphics[width=1.0\textwidth]{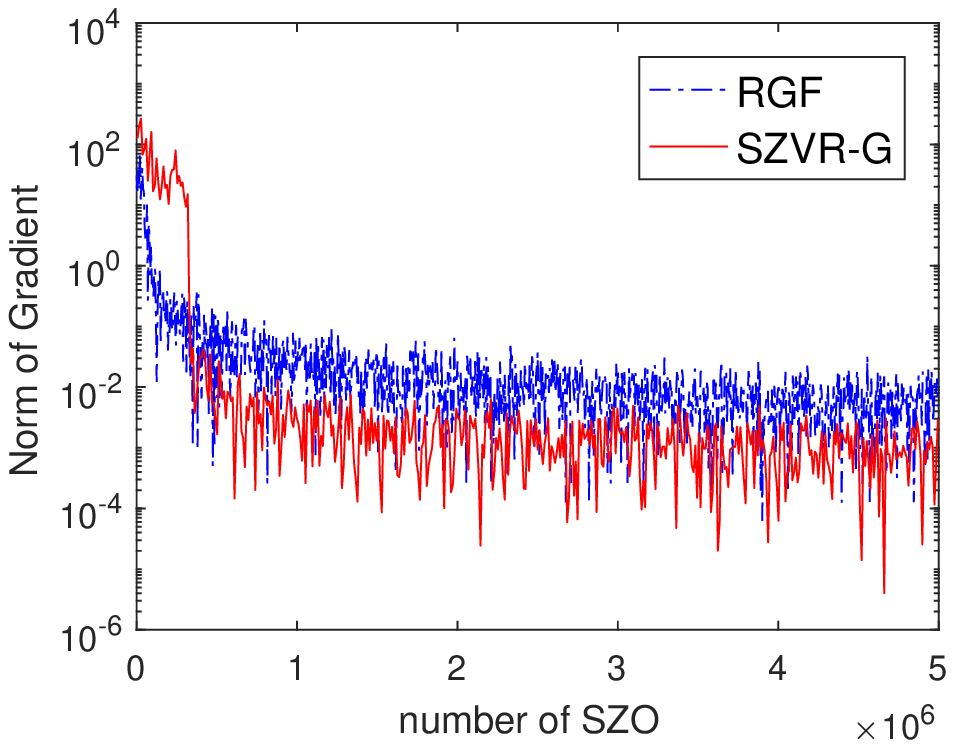}\\
			\includegraphics[width=1.0\textwidth]{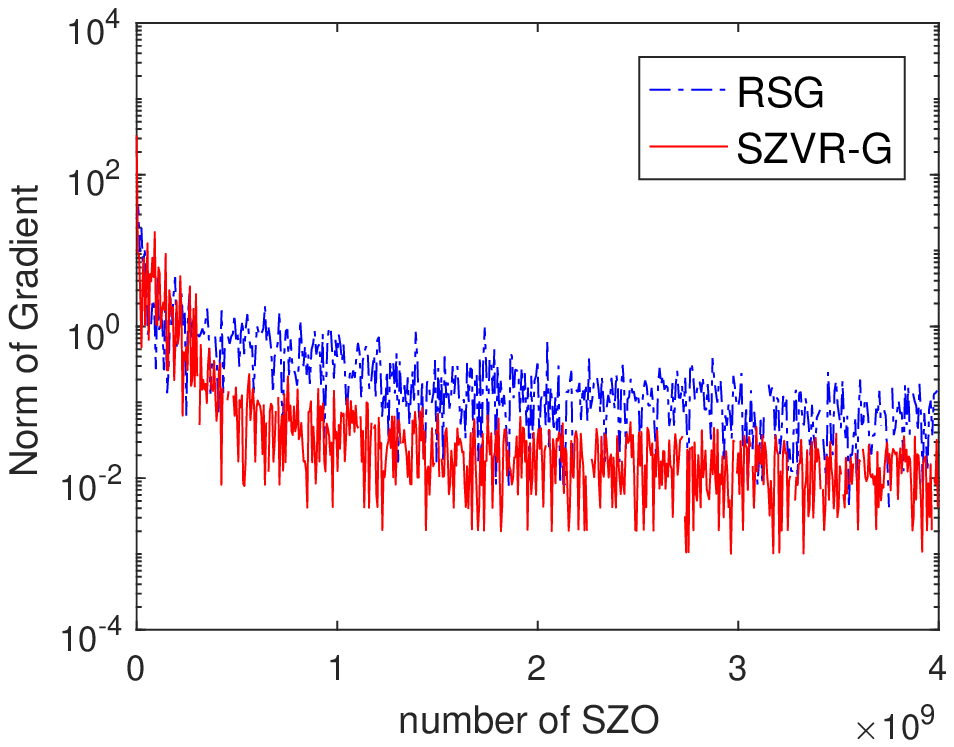}
		\end{minipage}
	}
	\caption{Comparison of different methods for  the objective function value 
		and the 
		norm of gradient. x-axis is the number of  
		$\mathcal{SZO}$. The first row shows the difference between SZVR-G and 
		RGF 
		for 
		the case of $n$=100. 	
		Note that we only use  
		$f(x)$, rather than $F(x,\xi)$ in both algorithms in order to verify 
		the variance 
		reduction technology in random direction vector. The second row shows 
		the difference between SZVR-G and RSG under the condition of $n$=2000. }
	\label{NS-SCSG-Zero:Figure:Exp-1-ab}
\end{figure}
In order to verify our theory, we apply our algorithm to logistic regression. 
Given $n$ training 
examples $\{ ( \xi _1, y_1 ),$ $( \xi _2, y_2 ),...,( \xi_n, y_n ) \}$, where 
$\xi_i  \in 
{\mathbb{R}^d}$ and $y_i$, $i\in [n]$ are the feature vector and the 
label of $i$th
example. 
The objective function is 
\begin{align*}
J( \theta  ) =  - \frac{1}{n}\left[ {\sum\nolimits_{i = 1}^n 
	{{y^{( i )}}\log {h_\theta }( {{\xi^{( i )}}} 
		) + ( {1 - {y^{( i )}}} )\log ( {1 - 
			{h_\theta }( {{\xi^{( i )}}} )} )} } \right],
\end{align*}
where ${h_\theta } = {1}/({{1 + {e^{ - {\theta ^T}\xi}}}})$. We use  
MNIST~\cite{lecun1998gradient} dataset to make two kinds of experiments in 
order to verify that our variance reduction technology is better than current 
approach. The 
dimension of $\theta$ is $d=14\times14\times10$, where the size of the image is 
$14\times14$, and the number of the class is 10. We choose the  parameters 
according to setting in  Theorem 
\ref{SCSG-Zero:SCSG:theorem:new:SZO-Complexity} to 
give the best performance. 
First, to verify that our variance reduction technique for Gaussian random directions
are useful, we compare our algorithm with RGF~\cite{nesterov2017random} for solving a 
deterministic function $f(x)$, which is the logistic regression with $n=100$ MNIST samples.
Row 1 in Figure~\ref{NS-SCSG-Zero:Figure:Exp-1-ab} shows the results that our method SZVR-G
is better than RGF~\cite{nesterov2017random}
both on the objective function value and the norm of 
the gradient. This verified that even for solving a deterministic function, 
our algorithm outperforms RGF in both theory and practice, due to the variance reduction for 
Gaussian search directions. 
%
%The 
%first 
%one is to compare our method with RGF \cite{nesterov2017random} that only 
%consider the Gaussian random vector. We use a fix
%small dataset, which are randomly selected  n=100 examples prior, to verity the 
%influence of 
%the variance 
%reduction on random direction. In particular, we 
%compute the whole loss function directly without selecting the example 
%randomly,  
%that is we consider $f(x)$ rather than $F(x,\xi)$. Figure 
%\ref{NS-SCSG-Zero:Figure:Exp-1-ab}. row 1 shows the results that  our method 
%SZVR-G is better than RGF 
%\cite{nesterov2017random} both on the objective function value and the norm of 
%the gradient. 

In the second experiment we compare with RSG 
\cite{ghadimi2013stochastic} on stochastic optimization, that consider two kinds of 
stochastic process: randomly select one or block example and Gaussian vector
to estimate the gradient. We use the  fix
dataset with randomly selected  $n=2000$ examples. Figure 
\ref{NS-SCSG-Zero:Figure:Exp-1-ab}. row 2 shows that  our method is  better  
than  RSG since we conduct variance reduction on both examples and Gaussian vectors.

%\begin{figure}
%	\centering
%	\subfigure{
%		\begin{minipage}[b]{0.35\textwidth}
%			\includegraphics[width=1.0\textwidth]{PlotEx1-b2.eps}
%		\end{minipage}
%	}
%	\subfigure{
%		\begin{minipage}[b]{0.35\textwidth}
%			\includegraphics[width=1.0\textwidth]{PlotEx1-b1.eps}
%		\end{minipage}
%	}
%	\caption{Comparison of 
%		our method (SZVR-G) and RSG for  the value of  
%		function and the norm of gradient. x-axis is the number of the 
%		$\mathcal{SZO}$ divided by the 
%		number of examples (we set n=2000). }
%	\label{NS-SCSG-Zero:Figure:Exp-1-b}
%\end{figure}

\subsection{Universal adversarial examples with black-box setting}

\begin{figure}
	\centering
	\subfigure{
		\begin{minipage}[b]{0.35\textwidth}
			\includegraphics[width=1.0\textwidth]{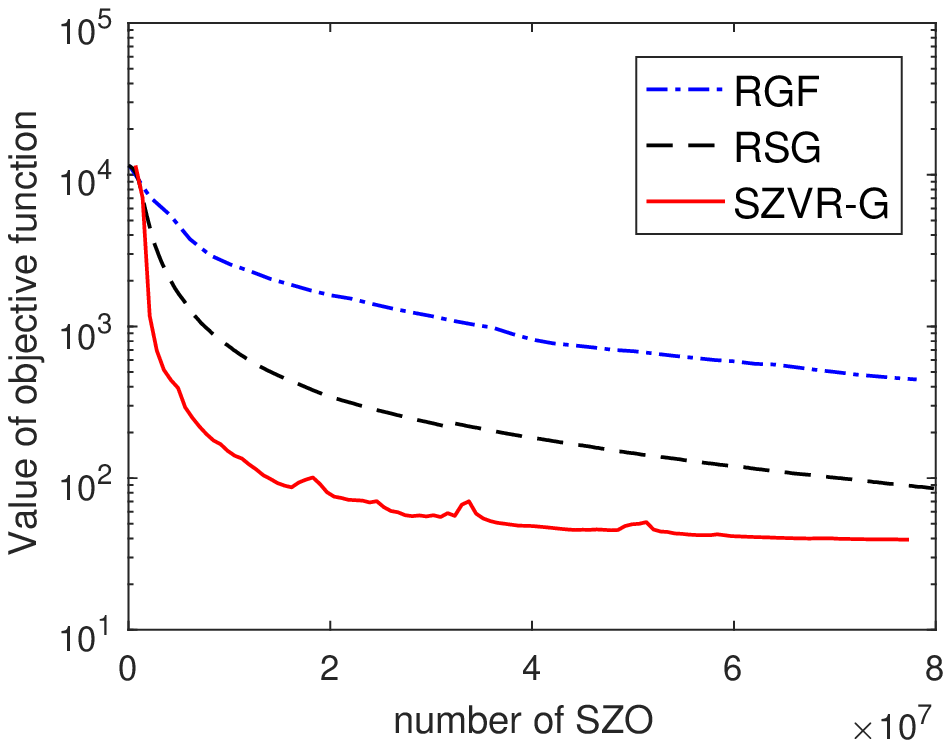}\\
			\includegraphics[width=1.0\textwidth]{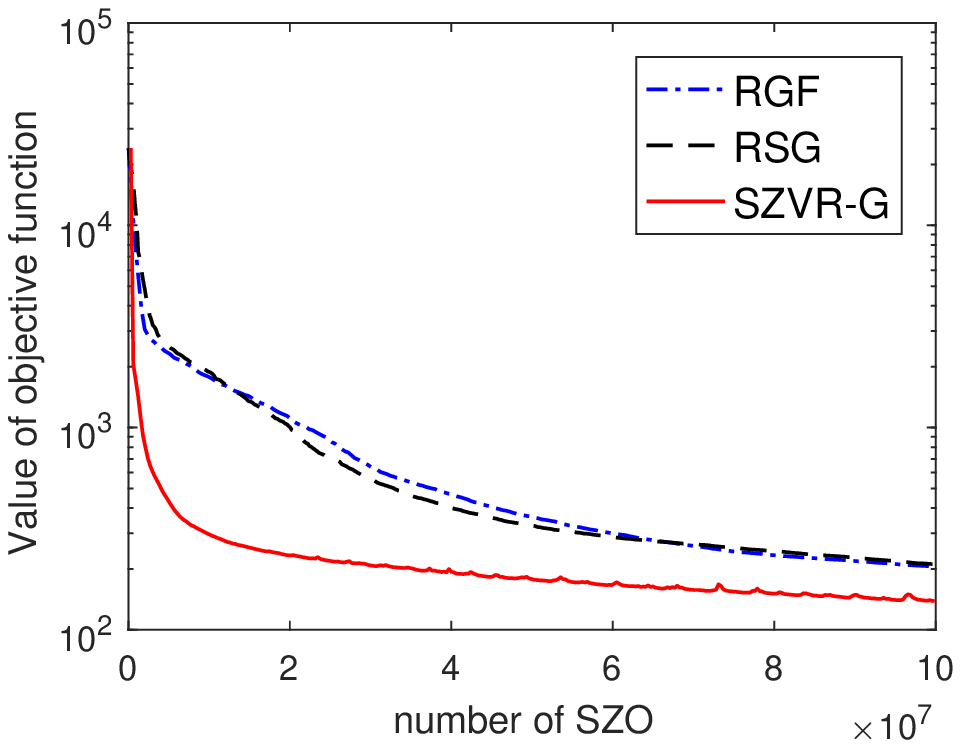}
		\end{minipage}
	}
	\subfigure{
		\begin{minipage}[b]{0.35\textwidth}
			\includegraphics[width=1.0\textwidth]{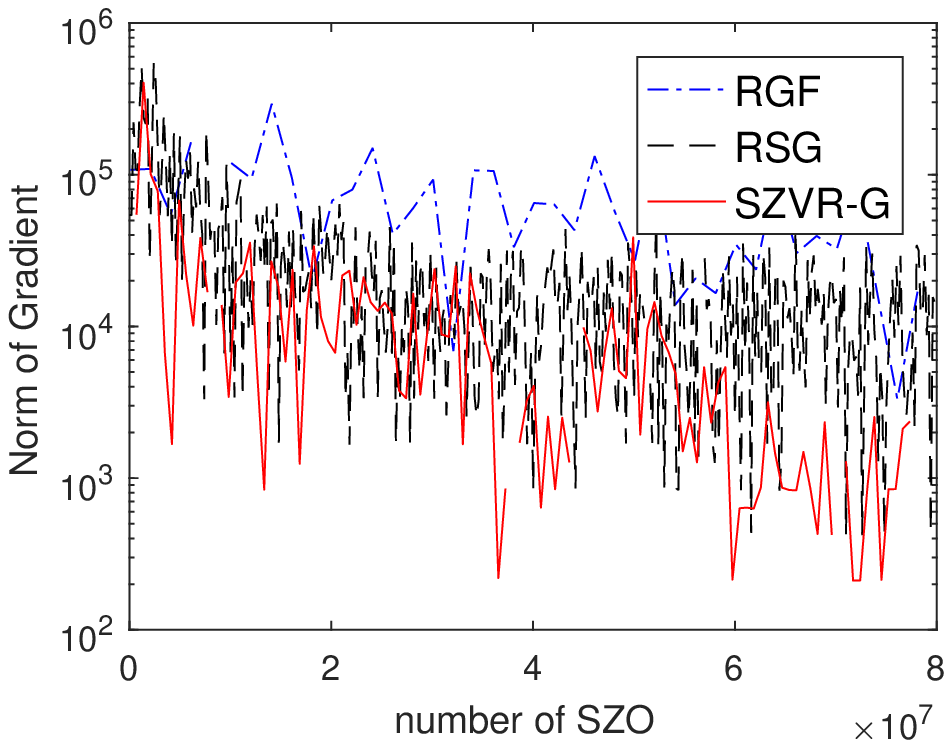}\\
			\includegraphics[width=1.0\textwidth]{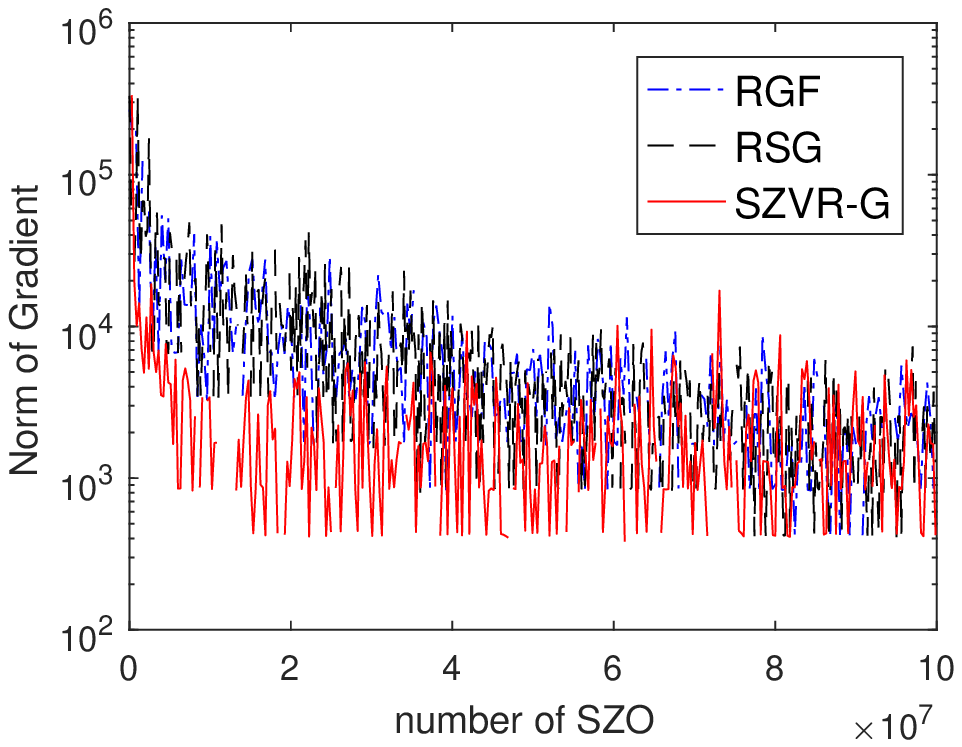}
		\end{minipage}
	}
	\caption{Comparison of RGF, RSG and our SZVR-G for  the objective function 
		value and the 
		norm of gradient, x-axis is the number of the 
		$\mathcal{SZO}$. Dateset: CIFAR-10 (Line 1) and MNIST (Line 2) }
	\label{NS-SCSG-Zero:Figure:Exp2-data1}
\end{figure}

In the second set of experiments, we apply zeroth order optimization methods 
to solve a real problem in adversarial black-box attack to machine learning models. 
It has been observed recently that convolutional neural networks are vulnerable 
to adversarial example ~\cite{szegedy2013intriguing,goodfellow2014explaining}. 
~\cite{chen2017zoo} apply zeroth order optimization techniques in the black-box 
setting, where one can only acquire input-output correspondences of targeted 
model. Also,  ~\cite{moosavi2017universal} finds there exists universal 
perturbations that could fool the classifier on almost all datapoints sampled. 
Therefore, we decide to apply our SZVR-G algorithm to non-smooth function that  
find universal 
adversarial perturbations in the black-box setting to show our efficiency in an 
interesting application. 
For classification models in neural networks, given the classification model 
$f: \mathbb{R}^d \rightarrow \{1, \dots, K\}$, it is usually assumed that 
$f(x)= \argmax_i(Z(x)_i)$, where $Z(x)\in\mathbb{R}^K$ is the final  
layer output, and $Z(x)_i$ is the prediction score for the $i$-th class.
Formally, we want to find a universal perturbation $\theta$ that could fool all 
N images in samples set $\Omega=\{(x_1,l_1),(x_2,l_2),\dots,(x_N,l_N)\}$, that 
is,
\begin{equation*}
\argmin_\theta L(\theta) = C \sum\nolimits_{i=1}^N 
\max\{[Z(x_i+\theta)]_{l_i}-\max_{j\neq l_i}[Z(x_i+\theta)]_j,-\kappa\} + 
||\theta||_2^2,
\label{eq:uni_untargeted_loss}
\end{equation*}
where $C$ is a constant to balance the distortion and attack success rate and 
$\kappa \geq 0$ is a confidence parameter that guarantees a constant gap 
between $\max_{j\neq l_i}[Z(x_i+\theta)]_j$ and $[Z(x_i+\theta)]_{l_i}$.
In this experiments, we use two standard datasets: 
MNIST~\cite{lecun1998gradient}, CIFAR-10~\cite{krizhevsky2009learning}. We 
construct two convolution neural networks following ~\cite{carlini2017towards}. 
In detail, both MNIST and CIFAR use the same network structure with four 
convolution layers, two max-pooling layers and two fully-connected layers. 
Using the parameters provided by ~\cite{carlini2017towards}, we could achieve 
99.5\% accuracy on MNIST and 82.5\% accuracy on CIFAR-10. All models are 
trained using Pytorch\footnote{https://github.com/pytorch/pytorch}. The 
dimension of $\theta$ is $d=28\times28$ for MNIST and $d=3\times32\times32$ for 
CIFAR-10. 
We tune the best parameters to give the best performance. Figure 
\ref{NS-SCSG-Zero:Figure:Exp2-data1} 
%and Figure 
%\ref{NS-SCSG-Zero:Figure:Exp2-data2} (In Appendix 
%\ref{NS-SCSG-Zero:Appendix:Section-exp}) 
show the performance with difference 
methods. We can see that our algorithm SZVR-G is better than RGF and RSG both 
on objective value and the norm of the gradient.

\section{Conclusion}\label{SCSG-Zero: Section:con}
In this paper, we present stochastic zeroth-order optimization via variance 
reduction for both smooth 
and non-smooth non-convex problem. The stochastic process include two kinds of 
aspects: randomly select the sample and derivative of direction, respectively. 
We give the theoretical analysis of  $\mathcal{SZO}$ complexity, which is 
better 
than that of RGF and RSG. Furthermore, we also extend our algorithm to 
mini-batch, in which the $\mathcal{SZO}$ complexity is multiplying a smaller 
size of the mini-batch. Our experimental result also confirm our theory.

%\today

{\small
	\bibliographystyle{unsrt}%{apalike}%
	\bibliography{egbib}
}
%\newpage
\appendix
\section{Technical Lemma}

%\begin{lemma} \label{SCSG-Zero:Gaussian:Lemma:Property-fmu}
%	\cite{nesterov2017random} For $f_{\mu}(x)$ defined in 
%	(\ref{SCSG-Zero:Gaussian:Definition-fmu}), 
%	\begin{itemize}
%		\item If $f \in {C^{0,0}}$, then 
%		${f_\mu } \in {C^{1,1}}$ 	with  $L_1\left( {{f_\mu }} \right) = 
%		\frac{{{1}}}{\mu }d^{1/2}{L_0}$, and  $\left| {{f_\mu }\left( x 
%			\right) - f\left( x \right)} \right| 
%		\le \mu {L_0}{d^{1/2}}$.	
%		\item If $f \in C^{1,1}$, then 	${f_\mu } 
%		\in C^{1,1},{L_1}( {{f_\mu 
%		}} ) \le {L_1}( f )$, and 
%		\begin{align*}
%		{\| {\nabla {f }( {{x}} )} \|^2} \le& {\r1 
%			2}{\| 
%			{\nabla {f_\mu }( {{x}} )} \|^2} + \frac{1}{2}{\mu 
%			^2}{L_1^2}{( {d+6} )^3},\\
%		{\mathbb{E}_u}{\| {{G_\mu }( {x,u,\xi } ) - \nabla 
%				{f_\mu }( x )} \|^2} \le& \frac{\mu^2 
%		}{2}L_1^2{( {d 
%				+ 6} )^3} + 2( {d + 4} )\| {\nabla 
%			f( x 
%			)} \|^2.
%		\end{align*}
%	\end{itemize}
%\end{lemma}

\begin{lemma}\label{SCSG-Zero:Appendix:LemmaGeometriProgression}
	For the sequences that satisfy ${c_{k - 1}} = {c_k}Y + U$, where $Y>1$, 
	$U>0$, $k\ge 1$ and $c_0>0$, we can get the geometric progression
	
	\centering{${c_k} + \frac{U}{{Y - 1}} = \frac{1}{Y}\left( {{c_{k - 1}} + 
			\frac{U}{{Y - 1}}} \right),$}
	\leftline{then $c_k$ can be represented as decrease sequences,}
	\centering{${c_k} = {\left( {\frac{1}{Y}} \right)^k}\left( {{c_0} + 
			\frac{U}{{Y - 1}}} \right) - \frac{U}{{Y - 1}}.$}
	\leftline{Furthermore, if $c_K\ge 0$, we have}
	\centering{${c_0} = \frac{{U\left( {{Y^K} - 1} \right)}}{{Y - 1}}.$}
\end{lemma}
\begin{lemma}\label{SCSG-Zero:Appendix:P-Bound-Gaussian-distribution}
	\cite{nesterov2017random} For
	$u\in \mathbb{R}^d$ and $p\ge 0$,
	$\frac{1}{\kappa }\int\limits_E^{} {{{\left\| u \right\|}^p}{e^{ - 
				\frac{1}{2}{{\left\| u \right\|}^2}}}du}  \le {\left( {p + d} 
		\right)^{p/2}}$, where $\kappa  = \int\limits_E^{} {{e^{ - 
				\frac{1}{2}{{\left\| u 
						\right\|}^2}}}du} $.
\end{lemma} 

\begin{lemma}\label{SCSG-Zero:Appendix:Bound-Gradient-Differentiable}
	\cite{nesterov2017random} 
	If $f(x)$ is differentiable at $x$ then,
	\begin{align}
	{\mathbb{E}_u}[ {{{\left\| {f'\left(x \right) \cdot u} \right\|}^2}} 
	] = {\mathbb{E}_u}[ {\left\langle {\nabla f\left( x \right),u} 
		\right\rangle {{\left\| u \right\|}^2}} ] = \left( {d + 4} 
	\right){\left\| {\nabla f\left( x \right)} \right\|^2},
	\end{align}
	where $f'\left( x \right) = \mathop {\lim }\limits_{\alpha  \downarrow 0} 
	\frac{1}{\alpha }\left[ {f\left( {x + \alpha u} \right) - f\left( x 
		\right)} \right]$.
	
\end{lemma} 

\begin{lemma}\label{SCSG-Zero:Appendix:lemma-expectationSubset} 
	If $v_1,...,v_n\in \mathbb{R}^d$ satisfy 
	$\sum\nolimits_{i = 1}^n {{v_i}}  = \vec 0$, and $\cal B$ is a non-empty, 
	uniform 
	random subset of $[n]$, then
	\begin{align*}
	{\mathbb{E}_{\cal B}} {{{\left\| {\frac{1}{B}\sum\nolimits_{b \in 
						{\cal 
							B}} 
					{{v_b}} 
				} 
				\right\|}^2}}  \le \frac{{\mathbb{I}\left( {B < n} 
			\right)}}{B}\frac{1}{n}\sum\limits_{i = 1}^n {v_i^2}.
	\end{align*}
	Furthermore, if the elements in $\cal B$ are independent, then 
	\begin{align*}
	{\mathbb{E}_\mathcal{B}}{{{\left\| {\frac{1}{B}\sum\nolimits_{b \in 
						\mathcal{B}} {{v_b}} } 
				\right\|}^2}} = \frac{1}{{Bn}}\sum\limits_{i = 1}^n 
	{v_i^2}. 
	\end{align*}
\end{lemma}
\begin{proof} Based on the  $\sum\nolimits_{i = 1}^n {{v_i}}  = \vec 0$, and 
	permutation and combination, we have
	\begin{itemize}
		\item For the case that $\cal B$ is a non-empty, 
		uniform random subset of $[n]$, we 
		have 
		\begin{align}
		{\mathbb{E}_{\cal B}} {{{\left\| {\sum\nolimits_{b \in \mathcal{B}} 
						{{v_b}} } 	\right\|}^2}}  =& {\mathbb{E}_{\cal 
				B}}\left[ {\sum\nolimits_{b 
				\in {\cal B}} 	{{{\left\| 	{{v_b}} \right\|}^2}} } \right] 
		+ 
		\frac{1}{{C_n^B}}\sum\nolimits_{i 	\in [n]} 
		{\left\langle {{v_i},\frac{{C_{n - 1}^{B - 1}\left( {B - 1} 
						\right)}}{{n - 
						1}}\sum\nolimits_{i \ne j} {{v_j}} } \right\rangle 
		} 
		\nonumber\\
		\label{SCSG-Zero:Appendix:lemma-expectationSubset-equality1}
		=& B\frac{1}{n}\sum\limits_{i = 1}^n {v_i^2}  + \frac{{B\left( {B - 
					1} 
				\right)}}{{n\left( {n - 1} \right)}}\sum\nolimits_{i \in 
			[n]} 
		{\left\langle 
			{{v_i},\sum\nolimits_{i \ne j} {{v_j}} } \right\rangle } \\
		=& B\frac{1}{n}\sum\limits_{i = 1}^n {v_i^2}  + \frac{{B\left( {B - 
					1} 
				\right)}}{{n\left( {n - 1} \right)}}\sum\nolimits_{i \in 
			[n]} 
		{\left\langle 
			{{v_i}, - {v_i}} \right\rangle }\nonumber\\ 
		\label{SCSG-Zero:Appendix:lemma-expectationSubset-equality2}
		=& \frac{{B\left( {n - B} \right)}}{{\left( {n - 1} 
				\right)}}\frac{1}{n}\sum\limits_{i = 1}^n {v_i^2}\\
		\le& B\mathbb{I}\left( {B < n} \right)\frac{1}{n}\sum\limits_{i = 
			1}^n 
		{v_i^2}. 
		\nonumber
		\end{align}
		\item 	For the case that the elements in $\cal B$ are independent, we 
		have
		\begin{align}
		{\mathbb{E}_\mathcal{B}} {{{\left\| {\sum\nolimits_{b \in 
							\mathcal{B}} {{v_b}} } 
					\right\|}^2}}  =& {\mathbb{E}_\mathcal{B}}\left[ 
		{\sum\nolimits_{b \in \mathcal{B}} 
			{{{\left\| {{v_b}} 
						\right\|}^2}} } \right] + 
		2{\mathbb{E}_\mathcal{B}}\left[ 
		{\sum\nolimits_{1 
				\le b < B} 
			{\left\langle {{v_b},\sum\nolimits_{b < k \le B} {{v_k}} } 
				\right\rangle } } \right]\nonumber\\
		=& B\frac{1}{n}\sum\nolimits_{i = 1}^n {{{\left\| {{v_i}} 
					\right\|}^2}}  + 2{\mathbb{E}_\mathcal{B}}\left[ 
		{\sum\nolimits_{1 \le b 
				< B} 
			{\left\langle {\mathbb{E}\left[ v \right],\sum\nolimits_{b < k 
						\le 
						B} 
					{{v_k}} } 
				\right\rangle } } \right]\nonumber\\
		\label{SCSG-Zero:Appendix:lemma-expectationSubset-equality3}
		=& B\frac{1}{n}\sum\nolimits_{i = 1}^n {{{\left\| {{v_i}} 
					\right\|}^2}}  + B\left( {B - 1} \right){\left\| 
			{\mathbb{E}\left[ v 
				\right]} 
			\right\|^2}\\
		=& B\frac{1}{n}\sum\nolimits_{i = 1}^n {{{\left\| {{v_i}} 
					\right\|}^2}} 
		\nonumber.
		\end{align} 
	\end{itemize}
\end{proof}

\begin{lemma}\label{SCSG-Zero:Appendix:lemma-expectationSubset+2} 
	Consider that  $\cal B$ is a non-empty, uniform random 
	subset of [n] with $\left| 
	{\cal B} \right| = B$, and  the set $\cal D$ with $\left| 
	{\cal D} \right| = D$, if
	$\cal D$ is a non-empty set, in which each element in $\cal 
	D$ is \textbf{independent},  and 
	$\frac{1}{n}\sum\nolimits_i {{\mathbb{E}_u}\left[ {h\left( {{\xi _i},u} 
			\right)} \right]}  = \vec 0$, then
	\begin{align}
	\label{SCSG-Zero:Appendix:lemma-expectationSubset+2:result1}
	{\mathbb{E}_{\cal B,\cal D}}\left[ {{{\left\| 
				{\frac{1}{{BD}}\sum\nolimits_{b \in {\cal B}} 
					{\sum\nolimits_{j \in {\cal D}} {h\left( {{\xi 
									_b},{u_j}} \right)} } } 
				\right\|}^2}} \right] \le 
	\left( {\frac{1}{D} + \frac{{\mathbb{I}(B < n)}}{B}} 
	\right)\frac{1}{n}\sum\limits_{i 
		= 1}^n {{\mathbb{E}_u}{{\left\| {h\left( {{\xi 
							_i},u} \right)} \right\|}^2}}.	
	\end{align}
\end{lemma}
\begin{proof} 
	$\cal D$ is a non-empty set, in which each element in $\cal 
	D$ is \textbf{independent}. Consider the $\cal B$ as an element, and 
	based on the result in 
	Lemma 
	\ref{SCSG-Zero:Appendix:lemma-expectationSubset}, we have 
	\begin{align*}
	{\left\| {\sum\nolimits_{j \in \mathcal{D}} {\sum\nolimits_{b \in 
					\mathcal{B}} {h\left( 
					{{\xi _b},{u_j}} \right)} } } \right\|^2} =& 
	\sum\nolimits_{j \in 
		\mathcal{D}} 
	{{{\left\| {\sum\nolimits_{b \in \mathcal{B}} {h\left( {{\xi _b},{u_j}} 
						\right)} 
				} \right\|}^2}} \\& + 2\sum\nolimits_{1 \le j < D} 
	{\left\langle 
		{\sum\nolimits_{b \in \mathcal{B}} {h\left( {{\xi _b},{u_j}} 
				\right)} 
			,\sum\nolimits_{j \le k \le {D}} {\sum\nolimits_{b \in 
					\mathcal{B}} 
				{h\left( {{\xi 
							_b},{u_k}} \right)} } } \right\rangle }.
	\end{align*}
	Take the expectation with respect to $\cal B $ 
	and $\cal D$ for the last two terms, we have
	\begin{itemize}
		\item For the first term,
		\begin{align}
		\label{SCSG-Zero:Appendix:lemma-expectationSubset+2-equality1}
		{\mathbb{E}_{\mathcal{B},\mathcal{D}}}\left[ {\sum\nolimits_{j 
				\in \mathcal{D}} {{{\left\| 
						{\sum\nolimits_{b \in \mathcal{B}} {h\left( {{\xi 
										_b},{u_j}} 
								\right)} } 
						\right\|}^2}} } \right] \le& 
		{\mathbb{E}_{\mathcal{B},\mathcal{D}}}\left[ {\sum\nolimits_{j 
				\in \mathcal{D}} {B\sum\nolimits_{b \in \mathcal{B}} 
				{{{\left\| 
							{h\left( {{\xi 
										_b},{u_j}} \right)} \right\|}^2}} } 
		} \right]\\
		=& BD{\mathbb{E}_\mathcal{B}}\left[ {\sum\nolimits_{b \in 
				\mathcal{B}} 
			{{\mathbb{E}_u}{{\left\| 
						{h\left( {{\xi _b},u} \right)} \right\|}^2}} } 
		\right]\nonumber\\
		=& BD\frac{B}{n}\sum\limits_{i = 1}^n {{\mathbb{E}_u}{{\left\| 
					{h\left( 
						{{\xi _i},u} \right)} \right\|}^2}} \nonumber,
		\end{align}
		where 
		(\ref{SCSG-Zero:Appendix:lemma-expectationSubset+2-equality1}) is 
		based on the fact that ${\left( {\sum\nolimits_{i = 1}^n {{a_i}} } 
			\right)^2} \le 
		n\sum\nolimits_{i = 1}^n {a_i^2}$.
		\item For the second term, 
		\begin{align}
		&2{\mathbb{E}_{\mathcal{B},\mathcal{D}}}\left[ {\sum\nolimits_{1 
				\le j < D} {\left\langle {\sum\nolimits_{b \in \mathcal{B}} 
					{h\left( {{\xi 
								_b},{u_j}} \right)} ,\sum\nolimits_{j < k 
						\le D} {\sum\nolimits_{b 
							\in \mathcal{B}} {h\left( {{\xi _b},{u_k}} 
							\right)} } } 
				\right\rangle } } 
		\right]\nonumber\\
		\label{SCSG-Zero:Appendix:lemma-expectationSubset+2-equality2}
		=& D\left( {D - 1} \right){\mathbb{E}_\mathcal{B}}{\left\| 
			{\sum\nolimits_{b \in \mathcal{B}} 
				{{\mathbb{E}_u}\left[ {h\left( {{\xi _b},u} \right)} 
					\right]} } 
			\right\|^2}\\
		\label{SCSG-Zero:Appendix:lemma-expectationSubset+2-equality3}
		=& D\left( {D - 1} \right)B\frac{{n - B}}{{n - 
				1}}\frac{1}{n}\sum\limits_{i = 1}^n 
		{{\mathbb{E}_u}{{\left\| 
					{h\left( {{\xi 
								_i},u} \right)} \right\|}^2}} ,
		\end{align}
		where 
		(\ref{SCSG-Zero:Appendix:lemma-expectationSubset+2-equality2}) 
		follows from independent between $j$ and $k$, and based on 
		(\ref{SCSG-Zero:Appendix:lemma-expectationSubset-equality3}), 
		and 
		(\ref{SCSG-Zero:Appendix:lemma-expectationSubset+2-equality3}) 
		follows 
		from (\ref{SCSG-Zero:Appendix:lemma-expectationSubset-equality2}) 
		in 
		Lemma \ref{SCSG-Zero:Appendix:lemma-expectationSubset} and the fact 
		$\frac{1}{n}\sum\nolimits_i {{\mathbb{E}_u}\left[ {h\left( {{\xi 
						_i},u} 
				\right)} \right]}  = \vec 0$.
	\end{itemize}	
	Thus, we have the expectation with respect to $\cal B$ and $\cal D$, 
	\begin{align*}
	{\mathbb{E}_{\cal B,{\cal D}}}{\left\| {\sum\nolimits_{b \in \cal B} 
			{\sum\nolimits_{j 
					\in \cal D} 
				{h\left( {{\xi _b},{u_j}} \right)} } } \right\|^2}\le& 
	BD\left( {B + \left( {D - 1} \right)\frac{{n - B}}{{n - 1}}} 
	\right)\frac{1}{n}\sum\limits_{i = 1}^n {{\mathbb{E}_u}{{\left\| 
				{h\left( {{\xi 
							_i},u} \right)} \right\|}^2}}\\ \le& 
	BD(B+D\mathbb{I}(B<n))\frac{1}{n}\sum\limits_{i = 
		1}^n {{\mathbb{E}_u}{{\left\| {h\left( {{\xi 
							_i},u} \right)} \right\|}^2}}.
	\end{align*}
\end{proof}

\subsection{The model of Convergence analysis}
Before give the official proof, we give a simple model of convergence sequence, 
which is easily comprehensive. First, given two sequences,
\begin{align*}
\left\| {{x_{k + 1}} - \tilde x} \right\|^2 \le& a\left\| {{x_k} - \tilde 
	x} \right\|^2 + b\left\| {\nabla f\left( {{x_k}} \right)} \right\|^2;\\
f\left( {{x_{k + 1}}} \right) \le& f\left( {{x_k}} \right) - p\left\| 
{\nabla 
	f\left( {{x_k}} \right)} \right\|^2 + q\left\| {{x_k} - \tilde x}
\right\|^2.
\end{align*}
Define ${c_k} = \left( {q + {c_{k + 1}}a} \right)$, we can see that 
\begin{align*}
f\left( {{x_{k + 1}}} \right) + {c_{k + 1}}\left\| {{x_{k + 1}} - \tilde x} 
\right\|^2 \le & f\left( {{x_k}} \right) + \left( {q + {c_{k + 1}}a} 
\right) {\left\| {{x_k} - \tilde x} \right\|^2} - \left( 
{p - b{c_{k + 1}}} \right)\left\| {\nabla f\left( {{x_k}} \right)} 
\right\|^2\\
=& f\left( {{x_k}} \right) + {c_k}\left\| {{x_k} - \tilde x} \right\|^2 - 
\left( {p - b{c_{k + 1}}} \right)\left\| {\nabla f\left( {{x_k}} \right)} 
\right\|^2,
\end{align*}
if  parameters $a,b,p,q>0$ satisfy, $\forall k > 0$,
\begin{itemize}
	\item $c_0>0$ and  ${c_k}\ge c_{k+1}>0$ is a decrease sequence;
	\item  $p - b{c_{k + 1}}\ge p - b{c_{0}}>0$.
\end{itemize}
Thus, we can obtain 
\begin{align*}
\frac{1}{K}\sum\limits_{k = 0}^{K - 1} {\left\| {\nabla f\left( {{x_k}} 
		\right)} 
	\right\|_2^2}  \le \frac{{f\left( {{x_0}} \right) + {c_0}\left\| {{x_0} 
			- 
			\tilde x} \right\|_2^2 - \left( {f\left( {{x_K}} \right) + 
			{c_K}\left\| 
			{{x_K} - \tilde x} \right\|^2} \right)}}{{K\left( {p - 
			b{c_0}} 
		\right)}}.
\end{align*}
How to choose the parameters and how to compute the best complexity of 
iteration or the gradient will be 
based on the algorithm we proposed and the property of the function we use in the following section.

\section{Convergence proof for Smooth function with Gaussian 
	smooth}\label{NS-SCSG-Zero:Appendix:smooth}
\subsection{Algorithm: mini-batch SZVR-G}
We present our mini-batch SZVR-G here in Algorithm 
\ref{SCSG-Zero:SCSG:Algorithm:Gaussian:VR-DB-Block}.
\begin{algorithm}[t]
	\caption{Mini-batch Zeroth-order via variance reduction with Gaussian 
		smooth}
	\label{SCSG-Zero:SCSG:Algorithm:Gaussian:VR-DB-Block}
	\begin{algorithmic}
		\Require $K$, $S$, $\eta$ (learning rate), and $\tilde{x}_0$
		%		\ENSURE Initialize $\tilde{\theta_1}$.
		\For{$s =0,1, 2,\cdots,S-1$}
		\State \textbf{Independently} Generate  Gaussian vector 
		set $u_{\cal D}$ through 
		Gaussian random	vector generator  with 
		$D$ times, where  ${\cal D}$ is the index set.  \Comment{In practice, 
			store  Gaussian random vector seeds for each $s$th iteration.}
		\State Sample from $[n]$ to form mini-batch $\mathcal{B}$ with 
		$|\mathcal{B}|=B$.
		\State $x_0=\tilde x_s$
		\State $G = {G_\mu }( {x_0,{u_{\cal D}},\xi_{\mathcal{B}} } )$ 
		\Comment{(\ref{SCSG-Zero:SCSG:Definition-SmoothGradient-Gaussian-mu-block})}
		\For{$k =0,1,2,\cdots,K-1$}
		\State $\Lambda=0$
		\For{$t =0,1,2,\cdots,b_0-1$}
		\State
		Sample ${i}$ from $[n]$ and ${j}$ from $[D]$ 
		\State
		$\Lambda_{t+1} = \Lambda_t+{G_\mu }\left( {{x_k},{u_{{\cal D}\left[ j 
					\right]}},{\xi _i}} \right) - {G_\mu }\left( {{{\tilde 
					x}_s},{u_{{\cal 
						D}\left[ j \right]}},{\xi _i}} \right)$ 
		\EndFor
		\State  ${{\tilde \nabla }_k}=\Lambda_{b_0}/b_0+G$
		\State ${x_{k + 1}} = {x_k} - \eta {{\tilde \nabla }_k}$%\Comment{4 
		%Queries}
		\EndFor
		\State Update $\tilde{x}_{s+1}=x_K $
		\EndFor \\	
		\textbf{Output:}  $ x_k^s$ , $s \in \left\{ {1,...,S} \right\},k 
		\in \left\{ {1,...,K} \right\}$ 
	\end{algorithmic}
\end{algorithm}
\subsection{Convergence tool}

In this section, we focus on Algorithm 
\ref{SCSG-Zero:SCSG:Algorithm:Gaussian:VR-DB} 
that 
apply to Gaussian-smoothed function, and  mainly give the upper bounds 
for ${\mathbb{E}_u}\| {{G_\mu }( {{{\tilde x}_s},{u_{\cal D}},\xi_{\cal B} } 
	) - \nabla {f_\mu }( {{{\tilde x}_s}} )} \|^2$, 
$ \mathbb{E}_{i,j,u}\| {G_\mu 
}( 
{{x_k},{u_{{\cal D}[ j ]}},{\xi _i}} ) 
- 
{G_\mu }( 
{{\tilde x}_s},u_{{\cal D}[ j ]},\xi 
_i 
) 
\|^2  $ and 
${{\mathbb{E}_{i,j,u}}{{\| {{{\tilde \nabla 
					}_k}} \|}^2}} $, which are used for 
analyzing the 
convergence sequence. Note, we drop the superscript 
${i}$ and $k$ of $\xi$ and $u$, respectively, for 
focusing on a single epoch analysis.
\begin{lemma}\label{SCSG-Zero:SCSG:Lemma:new:Bound-E[x-Ex]-Onevariance}
	In Algorithm \ref{SCSG-Zero:SCSG:Algorithm:Gaussian:VR-DB}, for $F(x,\xi) 
	\in 
	C^{1,1}$ and ${{G_\mu 
		}\left( {{{ x}_k},{u_{\cal D}},{\xi _{{\mathcal{B}}}}} 
		\right)}$ defined in 
	(\ref{SCSG-Zero:SCSG:Definition-SmoothGradient-Gaussian-mu-block}), we have
	\begin{align*}
	{\mathbb{E}_{u,\cal B}}{\left\| {{G_\mu }\left( {{{ 
						x}_k},{u_{\cal 
						D}},\xi_{\mathcal{B}} } 
			\right) - 
			\nabla {f_\mu }\left( {{{ x}_k}} \right)} \right\|^2} \le 
	\left( {\frac{1}{D} + \frac{{\mathbb{I}(B < n)}}{B}} 
	\right)\left( {\frac{\mu^2 }{2}L_1^2{{\left( {d + 6} \right)}^3} + 
		2\left( {d + 4} \right)\left\| {\nabla f\left(  x_k \right)} 
		\right\|^2} 
	\right).
	\end{align*}
\end{lemma}
\begin{proof} By the definition of  ${{G_\mu 
		}\left( {{{\tilde x}_s},{u_{\cal D}},{\xi _{\mathcal{B}}}} 
		\right)}$ defined in 
	(\ref{SCSG-Zero:SCSG:Definition-SmoothGradient-Gaussian-mu-block}), we have
	\begin{align}
	{\mathbb{E}_{u,\cal B}}{\left\| {{G_\mu }\left( {{{ 
						x}_k},{u_{\cal 
						D}},{\xi 
					_{\mathcal{B}}} } 
			\right) - \nabla {f_\mu }\left( {{{ x}_k}} \right)} 
		\right\|^2}
	\label{SCSG-Zero:SCSG:Lemma:Bound-E[x-Ex]-Onevariance-equality2}
	=& \left( {\frac{1}{D} + \frac{{\mathbb{I}(B < n)}}{B}} 
	\right)\frac{1}{n}\sum\limits_{i = 1}^n 
	{{\mathbb{E}_u}{{\left\| 
				{{G_\mu }\left( { x_k,u,{\xi _i}} \right) - \nabla {f_\mu 
					}\left( 
					{{{ x}_k}} \right)} \right\|}^2}}\\ 
	\label{SCSG-Zero:SCSG:Lemma:Bound-E[x-Ex]-Onevariance-equality3}
	\le& \left( {\frac{1}{D} + \frac{{\mathbb{I}(B < n)}}{B}} 
	\right)\left( {\frac{\mu^2 }{2}L_1^2{{\left( {d + 6} 
				\right)}^3} + 2\left( {d + 4} \right)\left\| {\nabla 
			f\left(  x_k 
			\right)} \right\|} \right),
	\end{align}
	where 
	(\ref{SCSG-Zero:SCSG:Lemma:Bound-E[x-Ex]-Onevariance-equality2}) based on 
	Lemma \ref{SCSG-Zero:Appendix:lemma-expectationSubset+2} that the vector in 
	$u_{\cal D}$ is independent, 
	(\ref{SCSG-Zero:SCSG:Lemma:Bound-E[x-Ex]-Onevariance-equality3}) is based 
	on the 
	Lemma \ref{SCSG-Zero:Gaussian:Lemma:Property-fmu}.
\end{proof}

The following Lemma can be obtained directly from Lemma 
\ref{SCSG-Zero:Appendix:lemma-expectationSubset} under the requirement that 
$\frac{1}{n}\sum\nolimits_{i = 1}^n {( {\nabla {f_\mu }( x ) - 
		\nabla {F_\mu }( {x,{\xi _i}} )} )}  = 0$, and Assumption 
\ref{SCSG-Zero:Assumption-smoothfunctinoVarianceBound}.
\begin{lemma}\label{SCSG-Zero:SCSG:Lemma:new:Bound-E[x-Ex]-F_mu}
	In Algorithm \ref{SCSG-Zero:SCSG:Algorithm:Gaussian:VR-DB}, for $F(x,\xi) 
	\in 
	C^{1,1}$ and $\nabla {f_\mu 
	}\left( 
	x \right)$ defined in 
	(\ref{SCSG-Zero:Gaussian:Definition-fmu}), we have
	\begin{align*}
	\r1 \frac{1}{B}\sum\nolimits_{i = 1}^B {{{\left\| {\nabla {f_\mu 
					}\left( x 
					\right) - \nabla {F_\mu }\left( {x,{\xi 
							_{\mathcal{B}\left[ i 
								\right]}}} \right)} 
				\right\|}^2}} 
	\le \frac{{\mathbb{I}(B < n)}}{B} H.
	\end{align*}
\end{lemma}

\begin{lemma}\label{SCSG-Zero:SCSG:Lemma:new:Bound-E[G_mu-G_mu]} 
	In Algorithm \ref{SCSG-Zero:SCSG:Algorithm:Gaussian:VR-DB},
	for $F(x,\xi) \in 
	C_{}^{1,1}$, 
	$\mu>0$, and ${{G_\mu }\left( {{x},{u},{\xi 
		}} \right)}$ defined in 
	(\ref{SCSG-Zero:Definition:Gaussian:SmoothGradient}), 
	we have
	\begin{align*}
	{{\mathbb{E}_{i,j,u}}{\left\| {{G_\mu }\left( {{x_k},{u_{\mathcal{D}\left[ j 
							\right]}},{\xi _i}} \right) - {G_\mu }\left( 
							{{{\tilde 
							x}_s},{u_{\mathcal{D}\left[ j \right]}},{\xi _i}} \right)} 
							\right\|^2}} 
	\le \frac{3}{2}L_1^2{\mu ^2}K{\left( {d + 6} \right)^3} + 3L_1^2{\left( 
		{d 
			+ 
			4} \right)}{{{\left\| {{x_k} - 
					{{\tilde 
							x}_s}} \right\|}^2}}.
	\end{align*}
\end{lemma}
\begin{proof} We first drop the subscript of $i$ 
	and ${{\cal 
			D}\left[ j 
		\right]}$
	in $\left\| {{G_\mu }\left( {{x_k},{u_{{\cal D}\left[ j \right]}},{\xi _i}} 
		\right) - {G_\mu }\left( {{{\tilde x}_s},{u_{{\cal D}\left[ j 
					\right]}},{\xi _i}} \right)} \right\|$ for simple and 
	easily understanding.	
	Through adding and subtracting the terms $ - \mu \left\langle {\nabla 
		F\left( {{{\tilde x}_s},\xi } \right),u} \right\rangle  + \mu 
	\left\langle 
	{\nabla F\left( {{{\tilde x}_s},\xi } \right),u} \right\rangle $, and by 
	the definition of ${G_\mu }\left( {{x_k},\xi ,{u}} \right)$ in 
	\ref{SCSG-Zero:Definition:Gaussian:SmoothGradient},  we have
	\begin{align*}
	&{\left\| {{G_\mu }\left( {{x_k},\xi ,{u}} \right) - {G_\mu }\left( 
			{{{\tilde x}_s},\xi ,{u}} \right)} \right\|^2}\\
	=& {\left\| {\frac{{F\left( {{x_k} + \mu u,\xi } \right) - F\left( 
					{{x_k},\xi } \right)}}{\mu }u - \frac{{F\left( 
					{{{\tilde 
								x}_s} + \mu 
						u,\xi } \right) - F\left( {{{\tilde x}_s},\xi } 
					\right)}}{\mu }u} 
		\right\|^2}\\
	=&\frac{{{{\| u \|}^2}}}{{{\mu ^2}}}{\left( {F\left( {{x_k} 
				+ \mu u,\xi } \right) - F\left( {{x_k},\xi } \right) - 
			\left( {F\left( 
				{{{\tilde x}_s} + \mu u,\xi } \right) - F\left( {{{\tilde 
							x}_s},\xi } 
				\right)} \right)} \right)^2}\\
	=& \frac{{{{\left\| u \right\|}^2}}}{{{\mu ^2}}}(F\left( {{x_k} + \mu 
		u,\xi } \right) - F\left( {{x_k},\xi } \right) - {\r1 \mu 
		\left\langle 
		{\nabla F\left( {{x_k},\xi } \right),u} \right\rangle } - \left( 
	{F\left( {{{\tilde x}_s} + \mu u,\xi } \right) - F\left( {{{\tilde 
					x}_s},\xi } \right) - {\r1 \mu \left\langle {\nabla 
				F\left( 
				{{{\tilde 
							x}_s},\xi } \right),u} \right\rangle }} 
	\right)\\
	&+ \r1 \mu \left\langle {\nabla F\left( {{x_k},\xi } \right),u} 
	\right\rangle  - \mu \left\langle {\nabla F\left( {{{\tilde x}_s},\xi } 
		\right),u} \right\rangle {)^2}\\
	\le& {\r1 3}\frac{{{{\left\| u \right\|}^2}}}{{{\mu ^2}}}{\left( 
		{F\left( 
			{{x_k} + \mu u,\xi } \right) - F\left( {{x_k},\xi } \right) - 
			\mu 
			\left\langle {\nabla F\left( {{x_k},\xi } \right),u} 
			\right\rangle } 
		\right)^2}\\& + {\r1 3}\frac{{{{\left\| u \right\|}^2}}}{{{\mu 
				^2}}}{\left( 
		{F\left( {{{\tilde x}_s} + \mu u,\xi } \right) - F\left( {{{\tilde 
						x}_s},\xi } \right) - \mu \left\langle {\nabla 
				F\left( {{{\tilde 
							x}_s},\xi } \right),u} \right\rangle } 
		\right)^2}\\
	&+ {\r1 3}\frac{{{{\left\| u \right\|}^2}}}{{{\mu ^2}}}{\left( {\mu 
			\left\langle {\nabla F\left( {{x_k},\xi } \right),u} 
			\right\rangle  - 
			\mu \left\langle {\nabla F\left( {{{\tilde x}_s},\xi } 
				\right),u} 
			\right\rangle } \right)^2}\\
	\le& 3\frac{{{{\left\| u \right\|}^2}}}{{{\mu ^2}}}\left( {{{\left( 
				{\frac{{{L_1}{\mu ^2}}}{2}{{\left\| u \right\|}^2}} 
				\right)}^2} + {{\left( 
				{\frac{{{L_1}{\mu ^2}}}{2}{{\left\| u \right\|}^2}} 
				\right)}^2} + 
		{{\left\langle {\nabla F\left( {{x_k},\xi } \right) - \nabla 
					F\left( 
					{{{\tilde x}_s},\xi } \right),u} \right\rangle }^2}} 
	\right)\\
	\le& \frac{3}{2}L_1^2{\mu ^2}{\left\| u \right\|^6} + 3{\left\langle 
		{\nabla F\left( {{x_k},\xi } \right) - \nabla F\left( {{{\tilde 
						x}_s},\xi } 
			\right),u} \right\rangle ^2}{\left\| u \right\|^2},
	\end{align*}
	where the first inequality follows from ${\left( 
		{{a_1} + {a_2} + {a_3}} \right)^2} \le 3a_1^2 + 3a_2^2 + 
	3a_3^2$, the second inequality is based on the smoothness of $F(x,\xi)$ and 
	$\left\langle {{b_1},{b_2}} \right\rangle  \le \left\| {{b_1}} 
	\right\|\left\| {{b_2}} \right\|$. the last inequality follows from 
	smoothness of $F(x,\xi)$.	Take expectation with respect to $j$, $i$ and 
	$u$, we have
	\begin{align*}
	&{{\mathbb{E}_{i,j,u}}{\left\| {{G_\mu }\left( {{x_k},{u_{\mathcal{D}\left[ j 
							\right]}},{\xi _i}} \right) - {G_\mu }\left( 
							{{{\tilde 
							x}_s},{u_{\mathcal{D}\left[ j \right]}},{\xi _i}} \right)} 
							\right\|^2}}\\  \le& 
	{{\mathbb{E}_{i,j,u}}\left[ {\frac{3}{2}L_1^2{\mu 
				^2}{{\left\| u \right\|}^6} + 3{\mu ^2}{{\left( 
					{\left\langle {\nabla 
							F\left( {{x_k},\xi } \right) - \nabla F\left( 
							{{{\tilde x}_s},\xi } 
							\right),u} \right\rangle } \right)}^2}} 
		\right]} \\
	\le &\frac{3}{2}L_1^2{\mu ^2}{\left( {d + 6} \right)^3} + 3\left( \r1{d 
		+ 
		4} 
	\right) {{{\left\| {\nabla F\left( {{x_k},\xi } 
					\right) - \nabla F\left( {{{\tilde x}_s},\xi } \right)} 
				\right\|}^2}} ,\\
	\le & \frac{3}{2}L_1^2{\mu ^2}{\left( {d + 6} \right)^3} + 3L_1^2\left( 
	\r1
	{d 	+ 	4} \right) {{{\left\| {{x_k} - {{\tilde 
							x}_s}} 	\right\|}^2}},
	\end{align*}
	where the second inequality is based on Lemma 
	\ref{SCSG-Zero:Appendix:P-Bound-Gaussian-distribution} for $p=6$ and $p=4$ 
	and \textbf{Lemma \ref{SCSG-Zero:Appendix:Bound-Gradient-Differentiable}, 
		which is an important lemma to strengthen the upper bound}; 
	The last inequality is based on the smoothness of $F$.
\end{proof}

\begin{lemma}\label{SCSG-Zero:SCSG:Lemma:new:Bound-Gradient:variance} 
	In Algorithm \ref{SCSG-Zero:SCSG:Algorithm:Gaussian:VR-DB}, for $F(x,\xi) 
	\in 
	C_{}^{1,1}$, 
	$\mu>0$, and ${{\tilde \nabla 
		}_k}$ defined in 
	(\ref{SCSG-Zero:SCSG:Definition-EstimateGradient-Gaussian}),   we have
	\begin{align*}
	{{\mathbb{E}_{i,j,u,\cal B}}{\| {{{\tilde \nabla 
						}_k}} \|^2}}  
	\le&  \frac{9}{2}L_1^2{\mu ^2}{\left( {d + 6} \right)^3} + 
	9L_1^2{\left(\r1  {d  
			+ 
			4} \right)} {{{\left\| {{x_k} - 
					{{\tilde 
							x}_s}} \right\|}^2}}+ 3 
	{{{\left\| 
				{\nabla {f_\mu }\left( {{x_k}} \right)} \right\|}^2}}\\
	&+   3\left( {\frac{1}{D} + \frac{{\mathbb{I}(B < n)}}{B}} 
	\right)\left( {\frac{\mu^2 }{2}L_1^2{{\left( {d + 6} \right)}^3} + 
		2\left( {d + 4} \right)\left\| {\nabla f\left(  x_k \right)} 
		\right\|^2} 
	\right).
	\end{align*}
\end{lemma}

\begin{proof}
	By adding and subtracting terms $\nabla {f_\mu }\left( {{x_k}} \right)$ and 
	$\nabla {f_\mu }\left( {{{\tilde x}_s}} \right)$, we have,	
	\begin{align}
		{\| {{{\tilde \nabla }_k}} \|^2} =& ||{G_\mu }\left( 
		{{x_k},{u_{\mathcal{D}[j]}},{\xi _i}} 
		\right) - {G_\mu }\left( {{{\tilde x}_s},{u_{\mathcal{D}\left[ j \right]}},{\xi _i}} 
		\right) + {G_\mu 
		}\left( {{{\tilde x}_s},{u_\mathcal{D}},{\xi _\mathcal{B}}} \right)\nonumber\\
		&- {G_\mu }\left( {{x_k},{u_\mathcal{D}},{\xi _\mathcal{B}}} \right) + {G_\mu }\left( 
		{{x_k},{u_\mathcal{D}},{\xi _\mathcal{B}}} 
		\right) - \nabla {f_\mu }\left( {{x_k}} \right) + \nabla {f_\mu }\left( {{x_k}} 
		\right)|{|^2}\nonumber\\
		=& ||{G_\mu }\left( {{x_k},{u_{\mathcal{D}[j]}},{\xi _i}} \right) - {G_\mu }\left( 
		{{{\tilde 
		x}_s},{u_{\mathcal{D}\left[ j \right]}},{\xi _i}} \right) - \left( {{G_\mu }\left( 
		{{x_k},{u_\mathcal{D}},{\xi 
		_\mathcal{B}}} \right) - {G_\mu }\left( {{{\tilde x}_s},{u_\mathcal{D}},{\xi _\mathcal{B}}} 
		\right)} 
		\right)\nonumber\\
		&+ \left( {{G_\mu }\left( {{x_k},{u_\mathcal{D}},{\xi _\mathcal{B}}} \right) - \nabla 
		{f_\mu }\left( 
		{{x_k}} 
		\right)} \right) + \nabla {f_\mu }\left( {{x_k}} \right)|{|^2}\nonumber\\
		\le& 3{\left\| {{G_\mu }\left( {{x_k},{u_{\mathcal{D}\left[ j \right]}},{\xi _i}} \right) - 
		{G_\mu 
		}\left( {{{\tilde x}_s},{u_{\mathcal{D}\left[ j \right]}},{\xi _i}} \right) - \left( 
		{{G_\mu }\left( 
		{{x_k},{u_\mathcal{D}},{\xi _\mathcal{B}}} \right) - {G_\mu }\left( {{{\tilde 
		x}_s},{u_\mathcal{D}},{\xi _\mathcal{B}}} \right)} 
		\right)} \right\|^2}\nonumber\\
	\label{SCSG-Zero:Lemma:new:Bound-Ksum-variance-equality1} 
		& + 3{\left\| {{G_\mu }\left( {{x_k},{u_\mathcal{D}},{\xi _\mathcal{B}}} \right) - \nabla 
		{f_\mu 
		}\left( 
		{{x_k}} \right)} \right\|^2} + 3{\left\| {\nabla {f_\mu }\left( {{x_k}} \right)} 
		\right\|^2},
	\end{align}	
	where (\ref{SCSG-Zero:Lemma:new:Bound-Ksum-variance-equality1}) is based on 
	the fact 
	that ${\left( 
		{{a_1} + {a_2} + {a_3}} \right)^2} \le 3a_1^2 + 3a_2^2 + 
	3a_3^2$.
	Taking expectation with respect to $i$, $j$ and $u$, we have
	\begin{align}
	{{\mathbb{E}_{i,j,u,\cal B}}{\| {{{\tilde 
							\nabla 
						}_k}} \|^2}}  \le& 3
	{{\mathbb{E}_{i,j,u}}{\left\| {{G_\mu }\left( {{x_k},{u_{\mathcal{D}\left[ j \right]}},{\xi 
	_i}} \right) 
	- {G_\mu }\left( {{{\tilde x}_s},{u_{\mathcal{D}\left[ j \right]}},{\xi _i}} \right) - \left( 
	{{G_\mu 
	}\left( {{x_k},{u_\mathcal{D}},{\xi _\mathcal{B}}} \right) - {G_\mu }\left( {{{\tilde 
	x}_s},{u_\mathcal{D}},{\xi _\mathcal{B}}} 
	\right)} \right)} \right\|^2}} \nonumber\\
	&+ 3{{{\left\| {\nabla {f_\mu }\left( 
					{{x_k}} \right)} \right\|}^2}}  + 3{\mathbb{E}_{u,\cal 
			B}}{\left\| 
		{{G_\mu 
			}\left( 
			{{{ x}_k},{u_{\cal D}},\xi_{\cal B} } \right) - \nabla 
			{f_\mu 
			}\left( 
			{{{ x}_k}} \right)} \right\|^2}\nonumber\\
	\label{SCSG-Zero:Lemma:new:Bound-Ksum-variance-equality2} 
	\le& 6 {{\mathbb{E}_{i,j,u}}{\left\| {{G_\mu }\left( 
				{{x_k},{u_{\mathcal{D}\left[ 
							j \right]}},{\xi _i}} \right) - {G_\mu }\left( 
							{{{\tilde 
							x}_s},{u_{\mathcal{D}\left[ j \right]}},{\xi _i}} 
							\right)} 
			\right\|^2}}  + 3 
	{{{\left\| 
				{\nabla {f_\mu }\left( {{x_k}} \right)} \right\|}^2}} \\
	&+ 3{\mathbb{E}_{u,\cal 
			B}}{\left\| {{G_\mu }\left( {{{ x}_k},{u_{\cal 
						D}},{\xi _{{\cal B}}} } 
			\right) - \nabla {f_\mu }\left( {{{ x}_k}} \right)} 
		\right\|^2}\nonumber\\
	\label{SCSG-Zero:Lemma:new:Bound-Ksum-variance-equality3} 
	\le& 9L_1^2{\mu ^2}{\left( {d + 6} \right)^3} + 
	18L_1^2{\left(\r1  {d 
			+ 
			4} \right)} {{{\left\| {{x_k} - 
					{{\tilde 
							x}_s}} \right\|}^2}}+ 3 
	{{{\left\| 
				{\nabla {f_\mu }\left( {{x_k}} \right)} \right\|}^2}}\\
	\label{SCSG-Zero:Lemma:new:Bound-Ksum-variance-equality4} 
	&+   3\left( {\frac{1}{D} + \frac{{\mathbb{I}(B < n)}}{B}} 
	\right)\left( {\frac{\mu^2 }{2}L_1^2{{\left( {d + 6} \right)}^3} + 
		2\left( {d + 4} \right)\left\| {\nabla f\left( { x}_k \right)} 
		\right\|^2} 
	\right),
	\end{align}
	where (\ref{SCSG-Zero:Lemma:new:Bound-Ksum-variance-equality2})	the 
	inequality follows from
	$ {\mathbb{E}_{i,\mathcal{B}}}{\left\| {{Z_i} - {Z_\mathcal{B}}} \right\|^2} \le 
	2{\mathbb{E}_i}{\left\| {{Z_i}} 
	\right\|^2} - 2{\left\| {{\mathbb{E}_i}{Z_i}} \right\|^2} \le 2{\mathbb{E}_i}{\left\| {{Z_i}} 
	\right\|^2} $, where $Z_i$ 
	is a random variable, $i\in[n]$, and random set $\mathcal{B} \subseteq  [n]$, ($i$ and 
	$\mathcal{B}$ are independent),
	${\mathbb{E}_\mathcal{B}}{Z_\mathcal{B}} = {\mathbb{E}_i}{Z_i}$; 
	(\ref{SCSG-Zero:Lemma:new:Bound-Ksum-variance-equality3}) and 
	(\ref{SCSG-Zero:Lemma:new:Bound-Ksum-variance-equality4}) are based on 
	Lemma 
	\ref{SCSG-Zero:SCSG:Lemma:new:Bound-E[G_mu-G_mu]} and Lemma 
	\ref{SCSG-Zero:SCSG:Lemma:new:Bound-E[x-Ex]-Onevariance}. Note that, for 
	convenience, we further take expectation with respect to $\cal B$ in the 
	last 
	equality.
\end{proof}
\subsection{Convergence analysis}

In this subsection, mainly based on  Lemma 
\ref{SCSG-Zero:SCSG:Lemma:new:Bound-Gradient:variance}, smoothness and update 
of 
$x$ in 
Algorithm \ref{SCSG-Zero:SCSG:Algorithm:Gaussian:VR-DB}, we give the new 
sequence of the proposed algorithm: ${\mathbb{E}_{i,j}}\left[ {{f_\mu }\left( 
	{{x_{k + 1}}} \right)} 
\right] + {c_{k 
		+ 1}}{\mathbb{E}_{i,j}}{\left\| {{x_{k + 1}} - {{\tilde x}_s}} 
	\right\|^2}$. In order to obtain the convergence sequence,we provide the 
formulation of the 
sequence $c_k$, $w_k$ and $J_k$, which is the key parameter in analyzing the 
convergence and complexity.  In 
Remark \ref{SCSG-Zero:Remark:parameter-C} and  
\ref{SCSG-Zero:Remark:parameter-u0}, we 
analyze the the parameter's 
relationship between $K$, $q$ and $\eta$ such that these new formed 
sequence can be converged.

\begin{lemma}\label{SCSG-Zero:SCSG:Lemma:new:sum-SK}
	In Algorithm 
	\ref{SCSG-Zero:SCSG:Algorithm:Gaussian:VR-DB},  
	for $F(x,\xi) \in 
	C_{}^{1,1}$, $\mu>0$, $q>0$, we have
	\begin{align*}
	&\frac{1}{S}\sum\limits_{s = 0}^{S-1} {\frac{1}{K}\sum\limits_{k = 0}^{K - 
			1} 
		{{{ w}_k}{{\left\| {\nabla f_{\mu}\left( {x_k^s} \right)} 
					\right\|}^2}} }  
	- \frac{1}{{SK}}\sum\limits_{s = 
		0}^{S - 
		1} {\sum\limits_{k = 0}^{K - 1} \beta_k{{{\left\| {\nabla f\left( {{{ x}_k}} 
						\right)} 
					\right\|}^2}} }
	\le \frac{R}{{SK}} + {J_{k + 1}},
	\end{align*}
	where   $x^*$ is the 
	optimal value of function $f_{\mu}(x)$, $R = {\max _x}\{ f_{\mu}( x ) - 
	f{_\mu}( x_* 
	):f{_\mu}( x ) \le f{_\mu}( x_0 ) \}$, and 
	\begin{align}
	\label{SCSG-Zero:Lemma:Bound:f+cx-ck}
	{c_k} =&   \left( {1+q\eta  + 
		18L_1^2\left( {d + 4} \right){\eta ^2}} \right){c_{k + 1}} + 
	9L_1^3\left( {d + 4} \right){\eta ^2}\\		
	{J_k} =&  \frac{3}{2}\left( {\frac{1}{D} + \frac{{\mathbb{I}(B < 
				n)}}{B} + 3} \right)\left( {{L_1} + 2{c_{k + 1}}} \right){\mu 
		^2}L_1^2{\left( {d + 6} \right)^3}{\eta ^2}\nonumber\\& + 
	\label{SCSG-Zero:Lemma:Bound:f+cx-Jk}
	\left( {1 + \frac{1}{q}{c_{k + 1}}} 
	\right)\frac{1}{2}\eta \frac{{\mathbb{I}(B < n)}}{B}H,\\		
	\label{SCSG-Zero:Lemma:Bound:f+cx-uk}	
	{{ w}_k} =& \left( {\frac{1}{2} - \frac{2}{q}{c_{k + 1}}} 
	\right)\eta  - 3\left( {\frac{{{L_1}}}{2} + {c_{k + 1}}} \right){\eta 
		^2},\\
	\label{SCSG-Zero:Lemma:Bound:f+cx-uk-beta}
	\beta_k  =& 6\left( {\frac{1}{D} + \frac{{\mathbb{I}(B < n)}}{B}} 
	\right)\left( {d + 
		4} 
	\right)\left( {\frac{{{L_1}}}{2} + {c_{k + 1}}} \right){\eta ^2}		
	\end{align}
\end{lemma}
\begin{proof}
	In each inner iteration, we use $x_k^s$ to indicate the iteration, but  we 
	drop the $s$ for simple. 
	Based on the smoothness of $f_{\mu}(x)$ and update of $x_k$ in Algorithm 
	\ref{SCSG-Zero:SCSG:Algorithm:Gaussian:VR-DB}, take 
	expectation 
	with respect to $i$, $j$ and $u$ and ${\mathbb{E}_{i,j,u}}[ {{{\tilde 
				\nabla 
			}_k}} ]$ in 
	\ref{SCSG-Zero:SCSG:Definition-EstimateGradient-Gaussian-Expect}, 
	we have,
	\begin{align*}
	&{\mathbb{E}_{i,j,u}}\left[ {{f_\mu }\left( {{x_{k + 1}}} \right)} \right]\\ 
	\le& 
	{f_\mu 
	}({x_k}) - \eta \langle \nabla {f_\mu }({x_k}),{\mathbb{E}_{i,j,u}}[{{\tilde 
			\nabla 
		}_k}]\rangle  + \frac{{{L_1}{\eta ^2}}}{2}{\mathbb{E}_{i,j,u}}{\| {{{\tilde 
					\nabla }_k}} \|^2}\\
	=& {f_\mu }\left( {{x_k}} \right) - \eta \left\langle {\nabla {f_\mu 
		}\left( {{x_k}} \right),\nabla {f_\mu }\left( {{x_k}} \right) - \nabla 
		{f_\mu }\left( {{{\tilde x}_s}} \right) + \nabla {F_\mu }\left( 
		{{{\tilde x}_s},{\xi _\mathcal{B}}} \right)} \right\rangle  + 
	\frac{{{L_1}{\eta 
				^2}}}{2}{\mathbb{E}_{i,j,u}}{\| {{{\tilde \nabla }_k}} \|^2}\\
	\le& {f_\mu }\left( {{x_k}} \right) - \eta {\left\| {\nabla {f_\mu 
			}\left( {{x_k}} \right)} \right\|^2} + \frac{1}{2}\eta {\left\| 
			{\nabla 
			{f_\mu }\left( {{x_k}} \right)} \right\|^2} + \frac{1}{2}\eta 
			\left\| 
	{\nabla {f_\mu }\left( {{{\tilde x}_s}} \right) + \nabla {F_\mu }\left( 
		{{{\tilde x}_s},{\xi _\mathcal{B}}} \right)} \right\| + 
	\frac{{{L_1}{\eta 
				^2}}}{2}{\mathbb{E}_{i,j,u}}{\| {{{\tilde \nabla }_k}} 
				\|^2};\\
	&{\mathbb{E}_{i,j,u}}{\| {{x_{k + 1}} - {{\tilde x}_s}} \|^2} \\= &
	{\mathbb{E}_{i,j,u}}{\| {{x_k} - \eta {{\tilde \nabla }_k} - {{\tilde 
					x}_s}} 
		\|^2} = {\left\| {{x_k} - {{\tilde x}_s}} \right\|^2} - 2\eta 
	\langle 
	{x_k} - {{\tilde x}_s},{\mathbb{E}_{i,j}}[{{\tilde \nabla }_k}]\rangle  + {\eta 
		^2}{\mathbb{E}_{i,j,u}}{\| {{{\tilde \nabla }_k}} \|^2}\\
	= &{\left\| {{x_k} - {{\tilde x}_s}} \right\|^2} - 2\eta \langle {x_k} - 
	{{\tilde x}_s},\nabla {f_\mu }\left( {{x_k}} \right) - \nabla {f_\mu 
	}\left( 
	{{{\tilde x}_s}} \right) + \nabla {F_\mu }\left( {{{\tilde x}_s},{\xi 
			_\mathcal{B}}} 
	\right)\rangle  + {\eta ^2}{\mathbb{E}_{i,j,u}}{\| {{{\tilde \nabla }_k}} 
		\|^2}\\
	\le& {\left\| {{x_k} - {{\tilde x}_s}} \right\|^2} + \eta q{\left\| {{x_k} 
	- 
			{{\tilde x}_s}} \right\|^2} + 2\frac{1}{q}\eta {\left\| {\nabla 
			{f_\mu }\left( 
			{{x_k}} \right)} \right\|^2} + 2\frac{1}{q}\eta {\left\| {\nabla 
			{f_\mu }\left( 
			{{{\tilde x}_s}} \right) + \nabla {F_\mu }\left( {{{\tilde 
			x}_s},{\xi 
					_\mathcal{B}}} 
			\right)} \right\|^2} + {\eta ^2}{\mathbb{E}_{i,j,u}}{\| {{{\tilde 
			\nabla }_k}} 
		\|^2},
	\end{align*}
	where the last equality is based on 	$2\left\langle 
	{{a_1},{a_2}} \right\rangle  \le \frac{1}{q}{\left\| {{a_1}} \right\|^2} + 
	q{\left\| {{a_2}} \right\|^2}$. 
	\textbf{	The key difference between the  proof of SVRG 
		\cite{reddi2016stochastic} is 
		the upper bound that separating the term $\nabla {f_\mu }\left( {{x_k}} 
		\right)$ and $ - 
		\nabla {f_\mu }\left( {{\tilde x_s}} \right) + \nabla {F_\mu 
		}\left( {{{\tilde x}_s},{\xi _{\cal B}}} \right)$, due to the fact that 
		they 
		are dependent with respect to $\cal B$}. For convenience, we take 
	expectation with respect to $\cal B$ and apply Lemma 
	\ref{SCSG-Zero:SCSG:Lemma:new:Bound-E[x-Ex]-F_mu}, thus, we have
	\begin{align*}
	{\mathbb{E}_{i,j,u}}\left[ {{f_\mu }\left( {{x_{k + 1}}} \right)} \right] \le& 
	{f_\mu }\left( {{x_k}} \right) -  \frac{1}{2}\eta {\left\| {\nabla 
			{f_\mu 
			}\left( {{x_k}} \right)} \right\|^2} + \frac{1}{2}\eta 
	\frac{{\mathbb{I}(B < 
			n)}}{B}H + \frac{{{L_1}{\eta ^2}}}{2}{\mathbb{E}_{i,j,u}}{\| {{{\tilde 
					\nabla }_k}} \|^2};\\
	{\mathbb{E}_{i,j,u}}{\left\| {{x_{k + 1}} - {{\tilde x}_s}} \right\|^2} \le& 
	{\left\| {{x_k} - {{\tilde x}_s}} \right\|^2} + \eta q{\left\| {{x_k} - 
			{{\tilde x}_s}} \right\|^2} + 2\frac{1}{q}\eta {\left\| {\nabla 
			{f_\mu 
			}\left( {{x_k}} \right)} \right\|^2} + 2\frac{1}{q}\eta 
	\frac{{\mathbb{I}(B < 
			n)}}{B}H + {\eta ^2}{\mathbb{E}_{i,j,u}}{\| {{{\tilde \nabla }_k}} 
		\|^2}
	\end{align*}

	Plus ${\mathbb{E}_{i,j,u}}\left[ {{f_\mu }\left( {{x_{k + 1}}} \right)} 
	\right] 
	$ with $ {c_{k + 1}}{\mathbb{E}_{i,j,u}}{\left\| {{x_{k + 1}} - {{\tilde 
					x}_s}} 
		\right\|^2}$, and apply Lemma 
	\ref{SCSG-Zero:SCSG:Lemma:new:Bound-Gradient:variance}, we have 
	\begin{align*}
	&{\mathbb{E}_{i,j,u}}\left[ {{f_\mu }\left( {{x_{k + 1}}} \right)} 
	\right] + {c_{k 
			+ 1}}{\mathbb{E}_{i,j,u}}{\left\| {{x_{k + 1}} - {{\tilde 
					x}_s}} 
		\right\|^2}\\ 
	\ge& {f_\mu }\left( {{x_k}} \right) + {c_k}{\left\| {{x_k} - {{\tilde 
					x}_s}} \right\|^2} - {w_k}{\left\| {\nabla{f_\mu }\left( 
			{{x_k}} 
			\right)} 
		\right\|^2}+J_k + \beta_k{\left\| {\nabla 
			f\left( {{ x}_k} \right)} \right\|^2},
	\end{align*}
	where
	\begin{align*}
	{c_k} =&   \left( {1+q\eta  + 
		18L_1^2\left( {d + 4} \right){\eta ^2}} \right){c_{k + 1}} + 
	9L_1^3\left( {d + 4} \right){\eta ^2}\\
	{{ w}_k} =& \left( {\frac{1}{2} - \frac{2}{q}{c_{k + 1}}} 
	\right)\eta  - 3\left( {\frac{{{L_1}}}{2} + {c_{k + 1}}} \right){\eta 
		^2},\\
	{{ J}_k} =& \frac{3}{2}\left( {\frac{1}{D} + \frac{{\mathbb{I}(B < 
				n)}}{B} + 3} 
	\right)\left( {{L_1} + 2{c_{k + 1}}} \right){\mu ^2}L_1^2{\left( {d + 
			6} \right)^3}{\eta ^2} + 
	\left( {1 + \frac{1}{q}{c_{k + 1}}} 
	\right)\frac{1}{2}\eta \frac{{\mathbb{I}(B < n)}}{B}H,\\
	\beta_k  =& 6\left( {\frac{1}{D} + \frac{{\mathbb{I}(B < n)}}{B}} 
	\right)\left( 
	{d + 4} 
	\right)\left( {\frac{{{L_1}}}{2} + {c_{k + 1}}} \right){\eta ^2}.
	\end{align*}

	Summing up ${\mathbb{E}_{i,j}}\left[ {{f_\mu }\left( {{x_{k }}} \right)} 
	\right] 
	+ {c_{k }}{\mathbb{E}_{i,j}}{\left\| {{x_{k}} - {{\tilde x}_s}} 
		\right\|^2}$ 
	from $k=0$ to $k=K-1$, we have
	\begin{align*}
	&{{\mathbb{E}_{i,j,u}}\left[ {{f_\mu }\left( {{x_{K}}} \right)} \right] + 
		{c_{K}}{\mathbb{E}_{i,j,u}}{{\left\| {{x_{K}} - {{\tilde x}_s}} 
				\right\|}^2}}\\ \le& {f_\mu }\left( {{x_0}} \right) + 
	{c_0}{\left\| 
		{{x_0} - {{\tilde x}_s}} \right\|^2} + K{J_k}  
	+\sum\limits_{k = 0}^{K - 1} { - {{ w}_k}{{\left\| {\nabla {f_\mu 
					}\left( {{x_k}} \right)} \right\|}^2}}  + 
	\sum\limits_{k = 0}^{K - 1} 
	\beta_k{{{\left\| {\nabla f\left( { x}_k \right)}\right\|}^2}}. 
	\end{align*}

	Thus, based on $x_0=\tilde{x}_s$ and $\tilde{x}_{s+1}=x_K$, we have
	\begin{align*}
	&\frac{1}{K}\sum\limits_{k = 0}^{K - 1} {{ w_k}{{\left\| {\nabla 
					{f_{\mu}
					}\left( {{x_k}} \right)} \right\|}^2}}-\frac{1}{K}\sum\limits_{k = 0}^{K - 1} 
					{{\beta _k}{\left\| {\nabla f\left( {{{ 
						x}_k}} \right)} \right\|^2}}  
					\\ 
	\le& 
	\frac{{{f_\mu 
			}\left( 
			{{x_0}} \right) + {c_0}{{\left\| {{x_0} - {{\tilde x}_s}} 
					\right\|}^2} 
			- \left( {{\mathbb{E}_{i,j,u}}\left[ {{f_\mu }\left( {{x_K}} \right)} 
				\right] + 
				{c_K}{\mathbb{E}_{i,j,u}}{{\left\| {{x_K} - {{\tilde x}_s}} 
						\right\|}^2}} 
			\right)}}{K} + {J_{k + 1}} \\
	\le& \frac{{{f_\mu }\left( {{x_0}} \right) + {c_0}{{\left\| {{x_0} - 
						{{\tilde x}_s}} \right\|}^2} - {\mathbb{E}_{i,j,u}}\left[ 
			{{f_\mu }\left( 
				{{x_K}} \right)} \right]}}{K} + {J_{k + 
			1}}\\
	=& \frac{{{f_\mu }\left( {{ x_0}} \right) - {\mathbb{E}_{i,j,u}}\left[ {{f_\mu 
				}\left( {{{\tilde x}_{s + 1}}} \right)} \right]}}{K} + 
	{J_{k + 
			1}}.			
	\end{align*}
	Summing up from $s=0$ to $s=S-1$, (here we add the $s$ to $x_k$ to indicate 
	different epoch $s$), thus we have
	\begin{align*}
	&\frac{1}{S}\sum\limits_{s = 0}^{S-1} {\frac{1}{K}\sum\limits_{k = 0}^{K - 
			1} 
		{{{ w}_k}{{\left\| {\nabla f_{\mu}\left( {x_k^s} \right)} 
					\right\|}^2}} }  
	- \frac{1}{S}\sum\limits_{s = 0}^{S - 1} {\frac{1}{K}\sum\limits_{k = 0}^{K - 1} {{\beta 
	_k}{{\left\| {\nabla f\left( {x_k^s} \right)} \right\|}^2}} } \\
	\le& \frac{{{f_\mu }\left( {{x_0}} \right) - {\mathbb{E}_{i,j,u}}\left[ {{f_\mu 
				}\left( {{{\tilde x}_S}} \right)} \right]}}{{SK}} + {J_{k + 
			1}}\\
	\le& \frac{{{f_\mu }\left( {{x_0}} \right) - {f_\mu }\left( {{x_*}} 
			\right)}}{{SK}} + {J_{k + 1}}\\
	\le& \frac{R}{{SK}} + {J_{k + 1}},
	\end{align*}
	where   $x^*$ is the 
	optimal value of function $f_{\mu}(x)$, $R = {\max _x}\{ f_{\mu}( x ) - 
	f{_\mu}( x_* 
	):f{_\mu}( x ) \le f{_\mu}( x_0 ) \}$, 
	
\end{proof}
\subsubsection{Parameters analysis}

In order to satisfy the convergence with the sequence $\left\{ {{u_k}} 
\right\}$, $\left\{ {{c_k}} \right\}$,  and $\left\{ {{J_k}} \right\}$, we 
consider the parameters setting, 
such as $K$, $q$, $\eta$, 
which will be used to analyze the complexity of $\mathcal{SZO}$.

\begin{remark}\label{SCSG-Zero:Remark:parameter-C} 
	For $c_k$ in (\ref{SCSG-Zero:Lemma:Bound:f+cx-ck}). 
	Due to the fact that,
	\begin{align*}
	{c_k} = \left( 
	{1 + \eta q + 18L_1^2\left( {d + 4} \right){\eta ^2}} \right){c_{k + 1}} 
	+ 
	9L_1^3\left( {d + 4} \right){\eta ^2} 
	\ge \left( 
	{1 + q\eta  + 18{L_1^2}{{\left( {d + 4} \right)}}{\eta ^2}} \right){c_{k 
			+ 
			1}} \ge {c_{k + 1}},
	\end{align*}
	$c_k$ is decreasing function, that is $c_k\le c_0$. Then, we can see that 
	$w_k\ge w_0$, 
	$J_k\le J_0$. Thus, 
	we only consider $c_0$, $w_0$, and $J_0$. Here, analyze $c_0$,
	\begin{itemize}
		\item 	For $c_0$, based on  Lemma 
		\ref{SCSG-Zero:Appendix:LemmaGeometriProgression}, we have	
		\begin{align*}
		{c_k} = {\left( {\frac{1}{Y}} \right)^k}\left( {{c_0} + 
			\frac{U}{{Y - 1}}} \right) - \frac{U}{{Y - 1}},
		\end{align*}
		where $Y = 1 + q\eta  + 18{L_1^2}{\left( {d + 4} \right)^2}{\eta ^2},$ 
		and 
		$U = 9{L_1^3}{\left( {d + 4} \right)^2}{\eta ^2}.$	 	
		There exists a large $K>0$  such that $c_K=0$ and a 
		lower 
		bound 
		for 
		$c_0$, that 
		is, 
		\begin{align}\label{SCSG-Zero:Remark:parameter-u0:equality-c0<L1}
		{c_0} =  \frac{{U\left( {{Y^K} - 1} \right)}}{{Y - 1}} = 
		\frac{{9{L_1^3}{{\left( {d + 4} \right)}}{\eta 
					^2}}}{{q\eta  + 
				18{L_1^2}{{\left( {d + 4} \right)}}{\eta ^2}}}C\le 
		\frac{1}{2}L_1C,
		\end{align}
		where $
		C = {\left( {1 + q\eta  + 18{L_1^2}{{\left( {d + 4} \right)}}{\eta 
					^2}} 
			\right)^K} - 1$.
		\item For $C = {( {1 + q\eta  + 
				18{L_1^2}{{( {d + 4} 
						)}}{\eta ^2}} )^K} - 1$ defined above.
		In order to keep the upper bound of $c_0\nrightarrow +\infty$, we use 
		the 
		function with 
		the 
		special 
		structure 
		and 
		characteristic of 		
		${( {1 + {{{t^1}}}} )^{{t^2}}} \to 
		e$\footnote{Here'e' is the Euler number, 2.718 }, where 
		${{{t^1}}}{{{t^2}}} 
		\le 1, 0<{t^1} 
		< 1$. Thus, we consider the number of inner iteration 
		$K$ that 
		has the relationship with $q$ and $\eta$, that is
		\begin{align}
		\label{SCSG-Zero:Remark:parameter-C-equality-K}
		\r1 ( {q\eta  + 18{L_1^2}{{( {d + 4} 
					)}}{\eta ^2}} )K \le 1,q\eta+18{L_1^2}( d + 4 
		){\eta ^2}<1.
		\end{align}	
		Then, we can see that $ C \le e - 1\approx 1.7$, and $c_0\le L_1$.
	\end{itemize}
\end{remark}

\begin{remark}\label{SCSG-Zero:Remark:parameter-u0}	
	For $w_0 = ( {1/2 - 2{c_{0}/q}} )\eta  - \left( {3L_1/2 + 3{c_{0}}} 
	\right){\eta ^2}$ defined in (\ref{SCSG-Zero:Lemma:Bound:f+cx-uk}), 
	should be 
	positive. In order to have small $\eta$ such that $w_0=O(\eta)$, we 
	should 	require
	$\r1 ( {1/2 - 2{c_{0}/q}} )$ is positive, that is $\r1 
	{c_{0}/q}\le O(1)$. 		
	We consider ${c_0}/q$, and the relationship with $w_0$ and  
	$\eta$, 
	separately,
	\begin{itemize}
		\item For  ${c_0}/q$.  we 
		set $\r1 {c_0}/q\le1/4$ for convenience. 	
		Specifically, $ C \le e - 
		1\approx 1.7$.
		\begin{align}
		\label{SCSG-Zero:Remark:parameter-u0:equality-c0/q}
		\frac{1}{q}{c_0} = \frac{1}{q}\frac{{9L_1^3{{\left( {d + 
							4} 
						\right)}}{\eta ^2}}}{{q\eta  + 18L_1^2{{\left( {d + 
							4} 
						\right)}}{\eta 
					^2}}}C \le \frac{1}{q}\frac{{9L_1^3{{\left( 
						{d + 
							4} 
						\right)}}\eta }}{{q + 
				18L_1^2{{\left( {d + 4} \right)}}\eta }} \le \frac{1}{4}.
		\end{align}
		By arrangement, we require the function with respect to $q$,
		\begin{align*}
		{q^2} + 18{L^2}{\left( {d + 4} \right)}\eta q - 36{L^3}{\left( {d 
				+ 
				4} 	\right)}\eta  \ge 0,
		\end{align*}	
		
		Based on $\sqrt {a_1^2 + a_2^2}  \le {a_1} + {a_2}$, $({a_1},{a_2} 
		\ge 
		0)$,  the solution of above quadratic square, $q_2$, is	
		\begin{align*}
		{q_2} =& \frac{{ - 18{L_1^2}{{\left( {d + 4} \right)}}\eta  + 
				\sqrt 
				{{{\left( 	{18{L_1^2}{{\left( {d + 4} \right)}}\eta } 
							\right)^2}} + 144{L_1^3}{{\left( {d 
								+ 4} \right)}}\eta } }}{2}\\
		\le& \frac{{ - 18L_1^2{{\left( {d + 4} \right)}}\eta  + 
				18L_1^2{{\left( 
						{d + 4} \right)}}\eta  + \sqrt {144L_1^3{{\left( {d 
								+ 4} 
							\right)}}\eta } }}{2} = 6\sqrt {L_1^3{{\left( 
					{d + 4} \right)}}\eta 
		} .
		\end{align*}		
		We set 
		\begin{align}
		\label{SCSG-Zero:Remark:parameter-u0:equality-q}
		\r1 q = 6\sqrt {L_1^3{{(d + 4)}}\eta },
		\end{align}
		that satisfy (\ref{SCSG-Zero:Remark:parameter-u0:equality-c0/q}). Based 
		on the setting of $q$, we have
		\begin{itemize}
			\item 
			Consider the requirement in 
			(\ref{SCSG-Zero:Remark:parameter-C-equality-K}) that 
			$q\eta+18{L_1^2}( d + 4 
			){\eta ^2}<1$, we 
			require that
			\begin{align}
			\label{SCSG-Zero:Remark:parameter-u0:equality-eta}
			\r1 \eta  \le \frac{1}{18d^{1/2} {L_1} 
			}<O\left(\frac{1}{d^{1/3}}\right).
			\end{align}
			\item Furthermore, the inequality in 
			(\ref{SCSG-Zero:Remark:parameter-C-equality-K}) that $( {q\eta  + 
				18{L_1^2}{{( {d + 4} 
						)}}{\eta ^2}} )K \le 1$, and q in 
			(\ref{SCSG-Zero:Remark:parameter-u0:equality-q}), 
			we require that		
			\begin{align}
			\label{SCSG-Zero:Remark:parameter-u0:equality-1/K}
			K \le \frac{2}{{\max \left\{ {d{\eta ^2},q\eta } \right\}}} = 
			\frac{2}{{\max \left\{ {d{\eta ^2},{d^{0.5}}{\eta ^{1.5}}} 
					\right\}}}.
			\end{align}	
		\end{itemize}

		\item For $w_0$ and  $\eta$, with convenience, we set 
		\begin{align}\label{SCSG-Zero:Remark:parameter-u0:equality-u0>eta}
		{w_0} = \left( {\frac{1}{2} - \frac{1}{q}{c_0}} \right)\eta  - 
		\left( {{L_1} + 2{c_0}} \right){\eta ^2} \ge \frac{1}{4}\eta  - 
		\left( {{L_1} + 2{c_0}} \right){\eta ^2} \ge \frac{1}{8}\eta, 
		\end{align}
		where the first inequality is based on $\frac{1}{q}{c_0} \le 
		\frac{1}{4} \Rightarrow \frac{1}{2} - \frac{1}{q}{c_0} \ge 
		\frac{1}{2}$. In order to obtain ${w_0} \ge \frac{1}{8}\eta $, using 
		$c_0\le L_1$ in (\ref{SCSG-Zero:Remark:parameter-u0:equality-c0<L1}), 
		as 
		long as 
		\begin{align*}
		\eta  \le \frac{1}{{24{L_1}}} \le \frac{1}{{8\left( {{L_1} + 
					2{L_1}} 
				\right)}}\mathop  \le \limits^{{c_0} \le {L_1}} 
		\frac{1}{{8\left( {{L_1} + 
					2{c_0}} \right)}}.
		\end{align*} However, combing with the range of $\eta$ in
		(\ref{SCSG-Zero:Remark:parameter-u0:equality-eta}), we can see that the 
		above equality is success immediately.
	\end{itemize}
\end{remark}

\subsubsection{Convergence of $\| \nabla f(x) \|^2$ }

Based on the above setting analysis: 
(\ref{SCSG-Zero:Remark:parameter-C-equality-K}), 
(\ref{SCSG-Zero:Remark:parameter-u0:equality-q}), 
(\ref{SCSG-Zero:Remark:parameter-u0:equality-1/K}) and 
(\ref{SCSG-Zero:Remark:parameter-u0:equality-u0>eta}), we reconsider the new 
sequence. Moreover, based on the relationship of $\left\| {\nabla f\left( x 
	\right)} 
\right\|_{}^2$ and $\left\| {\nabla f_{\mu}\left( x \right)} \right\|_{}^2$ in 
Lemma \ref{SCSG-Zero:Gaussian:Lemma:Property-fmu} and the convergence of 
$\left\| {\nabla f_{\mu}\left( x \right)} \right\|_{}^2$, we present the 
convergence of $\left\| {\nabla f\left( x \right)} 
\right\|_{}^2$.

\begin{theorem2}\ref{SCSG-Zero:SCSG:Theorem:new:SZO-Convergence}
	In Algorithm 
	\ref{SCSG-Zero:SCSG:Algorithm:Gaussian:VR-DB},  under Assumption 
	\ref{SCSG-Zero:Assumption-smoothfunctinoVarianceBound},
	for $F(x,\xi) \in 
	C_{}^{1,1}$, let parameters $\mu, \eta, q,K>0$,  $c_0\le L_1$ and 
	the 
	cardinality of Gaussian vector set and sample set $D \ge O\left( {\eta d} 
	\right)$ and  $B \ge O\left( \min\{\eta d,n\} \right)$,  we have
	\begin{align*}
	\frac{1}{S}\sum\limits_{s = 0}^{S-1} {\frac{1}{K}\sum\limits_{k = 0}^{K - 
			1} 
		{{{\left\| {\nabla f\left( {x_k^s} \right)} \right\|}^2}} }  \le 
	\frac{{32R}}{{SK\eta }} + \frac{32}{\eta }{J_0}+ \frac{1}{2}{\mu 
		^2}L_1^2{(d 
		+ 6)^3},
	\end{align*}
	where   $x^*$ is the 
	optimal value of function $f_{\mu}(x)$, $R = {\max _x}\{ f_{\mu}( x ) - 
	f{_\mu}( x_* 
	):f{_\mu}( x ) \le f{_\mu}( x_0 ) \}$, and 
	\begin{align}
	\label{SCSG-Zero:Theorem:SZO-Convergence-Jk}
	{J_{0}} =\frac{3}{4}\left( {\frac{1}{D} + \frac{{\mathbb{I}(B < n)}}{B} 
		+ 3} 
	\right)\left( {{L_1} + 2{c_{0}}} \right)\mu L_1^2{\left( {d + 6} 
		\right)^3}{\eta ^2} + \left( {1 + \frac{1}{q}{c_{0}}} 
	\right)\frac{1}{2}\eta \frac{{\mathbb{I}(B < n)}}{B}H.
	\end{align}
\end{theorem2}

\begin{proof}
	Base on the parameters that $c_k$ is decreasing sequence and $c_0>c_k$, 
	$w_0<w_k$, and $-\beta_0<-\beta_k$,
	$J_k<J_0$, following Lemma \ref{SCSG-Zero:SCSG:Lemma:new:sum-SK}, we 
	have	
	\begin{align*}
	&	\frac{1}{S}\sum\limits_{s = 0}^{S-1} {\frac{1}{K}\sum\limits_{k = 0}^{K - 1} 
	{{w_0}{{\left\| 
	{\nabla {f_\mu }\left( {x_k^s} \right)} \right\|}^2}} }  -\frac{1}{S}\sum\limits_{s = 
	0}^{S - 1} {\frac{1}{K}\sum\limits_{k = 0}^{K - 1} {\beta _0}{{{\left\| {\nabla f\left( {x_k^s} 
	\right)} 
	\right\|}^2}} } \\
	\le& \frac{1}{S}\sum\limits_{s = 0}^{S-1} {\frac{1}{K}\sum\limits_{k = 0}^{K - 1} 
	{{w_k}{{\left\| 
	{\nabla {f_\mu }\left( {x_k^s} \right)} \right\|}^2}} }  - \frac{1}{{S}}\sum\limits_{s = 0}^{S 
	- 1}\frac{1}{{K}} {\sum\limits_{k = 0}^{K - 1} {{\beta _k}} {{\left\| {\nabla f\left( {{{ 
	x}^s_k}} 
	\right)} \right\|}^2}} \\
	\le& \frac{R}{{SK}} + {J_{k + 1}}\\
	\le& \frac{R}{{SK}} + {J_{0}},
	\end{align*}
	
	Combing with the inequality in Lemma 
	\ref{SCSG-Zero:Gaussian:Lemma:Property-fmu} that replace the smoothed 
	$\nabla f_{\mu}(x)$ with $\nabla f(x)$, we have
	\begin{align*}
	&\left( {\frac{1}{2}{w_0} - {\beta _0}} 
	\right)\frac{1}{{SK}}\sum\limits_{s = 0}^{S-1} {\sum\limits_{k = 0}^{K - 1} 
		{{{\left\| {\nabla f\left( {x_s^k} \right)} \right\|}^2}} }  - 
	\frac{1}{2}{w_0}{\mu ^2}L_1^2{(d + 6)^3}\\
	=& \frac{1}{{SK}}\sum\limits_{s = 0}^{S-1} {\sum\limits_{k = 0}^{K - 1} 
		{\frac{1}{2}{w_0}\left( {{{\left\| {\nabla f\left( {x_s^k} \right)} 
						\right\|}^2} - \frac{1}{2}{\mu ^2}L_1^2{{(d + 6)}^3}} 
						\right)} }  - 
	{\beta _0}\frac{1}{{SK}}\sum\limits_{s = 0}^{S - 1} {\sum\limits_{k = 
			0}^{K - 1} {{{\left\| {\nabla f\left( {x_s^k} \right)} 
			\right\|}^2}} } 
	\\
	=& {w_0}\frac{1}{{SK}}\sum\limits_{s = 0}^{S - 1} {\sum\limits_{k = 
			0}^{K - 1} {{{\left\| {\nabla f\left( {x_s^k} \right)} 
			\right\|}^2}} }  
	- {\beta _0}\frac{1}{{SK}}\sum\limits_{s = 0}^{S - 1} {\sum\limits_{k = 
			0}^{K - 1} {{{\left\| {\nabla f\left( {x_s^k} \right)} 
			\right\|}^2}} } 
	\\
	\le& \frac{R}{{SK}} + {J_0},
	\end{align*}
	In order to keep $\frac{1}{2}{w_0} - \beta_0> 0$, based on ${w_0} \ge 
	\frac{1}{8}\eta $ in (\ref{SCSG-Zero:Remark:parameter-u0:equality-u0>eta}), 
	we require   $\beta_0\le\eta/32$, based on the definition of $\beta_k$ in 
	(\ref{SCSG-Zero:Lemma:Bound:f+cx-uk-beta}), we have,
	\begin{align*}
	6\left( {\frac{1}{D} + \frac{{\mathbb{I}(B < n)}}{B}} \right)\left( {d 
		+ 4} 
	\right)\left( {\frac{{{L_1}}}{2} + {c_{k + 1}}} \right){\eta ^2} \le 
	\frac{1}{{32}}\eta, 
	\end{align*}
	thus, we require
	\begin{align}\label{SCSG-Zero:Theorem:SZO-Convergence-equality-D}
	D\ge O\left( {\eta d} \right),B \ge O\left( \min\{\eta d,n\} \right).
	\end{align}
	
	Base on the parameters setting in 
	(\ref{SCSG-Zero:Remark:parameter-u0:equality-u0>eta}) that replace $w_0$ 
	with 
	$\eta/16$, we have	
	\begin{align*}
	\frac{1}{S}\sum\limits_{s = 0}^{S-1} {\frac{1}{K}\sum\limits_{k = 0}^{K - 
			1} 
		{{{\left\| {\nabla f\left( {x_k^s} \right)} \right\|}^2}} }  \le 
	\frac{{32R}}{{SK\eta }} + \frac{32}{\eta }{J_0}+ \frac{1}{2}{\mu 
		^2}L_1^2{(d 
		+ 6)^3}.
	\end{align*}
	
\end{proof}

\subsection{Complexity analysis}

\begin{theorem2}\ref{SCSG-Zero:SCSG:theorem:new:SZO-Complexity} In Algorithm 
	\ref{SCSG-Zero:SCSG:Algorithm:Gaussian:VR-DB},  for $F(x,\xi) \in 
	C_{}^{1,1}$,  
	let the size of  sample set $\cal B$, $B$ =$O (\min \{ 
	n,1/{\varepsilon 
		^2} 
	\})$,  the step 
	$\eta  =O (\min \{ 1/(d^{2/3}B^{1/3}),1/(d^{1/3}B^{2/3}) \})$,  $\mu  \le
	O(\varepsilon /({L_1} d 
	^{1.5}))$, and the 
	number of 
	inner iteration $K \le O (1/\max \{ d\eta ^2,d^{0.5}\eta 
	^{1.5} \})$,   Gaussian vectors set  
	$D \ge O(\eta d)$.
	In 
	order to 
	obtain 
	\begin{align*}
	\frac{1}{S}\sum\nolimits_{s = 0}^{S - 1} {\frac{1}{K}\sum\nolimits_{k = 
			0}^{K - 1} {{{\left\| {\nabla f\left( {{x_k^s}} \right)} 
					\right\|}^2}} }  \le 
	{\varepsilon ^2},
	\end{align*}
	the  total number of $T_{\mathcal{SZO}}$ is at most 
	$O(\frac{1}{{{\varepsilon 
				^2}}}\max \left\{ {{d^{2/3}}{B^{1/3}},{d^{1/3}}{B^{2/3}}} 
	\right\})$, with the number of total iterations  $T>O(1/(\varepsilon^2 
	\eta))$.
\end{theorem2}
\begin{proof} Based on the results in Theorem 
	\ref{SCSG-Zero:SCSG:Theorem:new:SZO-Convergence}, in order 
	to obtain 
	\begin{align*}
	\frac{{32R}}{{SK\eta }} + \frac{32}{\eta }{J_0}+ \frac{1}{2}{\mu 
		^2}L_1^2{(d 
		+ 6)^3} 
	\le 
	\varepsilon^2,
	\end{align*}
	we separately  analysis to obtain the complexity:
	\begin{itemize}
		\item The first term: in order to obtain $\frac{1}{2}{\mu 
			^2}{L_1^2}{{\left( {d+6} 
				\right)}^3}	\le \frac{1}{3}\varepsilon^2 $, we have
		\begin{align}\label{SCSG-Zero:SCSG:theorem:SZO-Complexity:equality-}
		\r1 \mu  \le {{2\varepsilon }}/({{{L_1}{{\left( {d + 6} 
						\right)}^{3/2}}}}).
		\end{align}
		\item The second term: in order to obtain 
		\begin{align*}
		\frac{1}{\eta }{J_0} = \frac{3}{2}\left( {\frac{1}{D} + \frac{{(B < 
					n)}}{B} + 3} \right)\left( {{L_1} + 2{c_{0}}} \right){\mu 
			^2}L_1^2{\left( {d + 6} \right)^3}\eta  + \left( {1 + 
			\frac{1}{q}{c_{0}}} 
		\right)\frac{1}{2} \frac{{\mathbb{I}(B < n)}}{B}H \le 
		\frac{{{\varepsilon ^2}}}{3},
		\end{align*}
		based on $\mu$ in 
		(\ref{SCSG-Zero:SCSG:theorem:SZO-Complexity:equality-}),  $ \eta  \le 
		\frac{1}{d^{1/2} {L_1} }<1$ in 
		(\ref{SCSG-Zero:Remark:parameter-u0:equality-eta}), the first sub-term 
		is satisfied; based on 
		$c_0/q<1/4$	in (\ref{SCSG-Zero:Remark:parameter-u0:equality-c0/q}) and 
		the upper bound of $B$ in 
		(\ref{SCSG-Zero:Theorem:SZO-Convergence-equality-D}), we 
		require
		\begin{align}
		B \ge \min \left\{ {n,\frac{1}{{{\varepsilon ^2}}}} \right\}\& \min 
		\left\{ {n,\eta d} \right\} = \min \left\{ {\max \left\{ {\eta 
				d,\frac{1}{{{\varepsilon ^2}}}} \right\},n} \right\}.
		\end{align}	
		We consider the case \footnote{otherwise, d is much large than 
			$\frac{1}{\varepsilon^8}$. Because, we consider the case, if $n 
			\ge \eta 
			d \ge \frac{1}{{{\varepsilon ^2}}}$ , then the bast complexity is 
			${T_{SZO}} = 
			\frac{{{d^{3/4}}}}{{{\varepsilon ^2}}}$ when
			${\eta ^*} = \frac{1}{{{d^{3/4}}}}$,  which is from 
			(\ref{SCSG-Zero:SCSG:theorem:SZO-Complexity:equality-T-3}), then 
			$\eta 
			d = {d^{1/4}} \ge \frac{1}{{{\varepsilon ^2}}} \Rightarrow d \ge 
			\frac{1}{{{\varepsilon ^8}}}$, and also require  $n\ge 
			\frac{1}{{{\varepsilon ^8}}}$, if $\varepsilon  \le 0.01$, it will 
			be 
			the extreme case we do not 
			consider here. but even in that case, our result is also better 
			than 
			RGF and RSG.}
		\begin{align}\label{SCSG-Zero:SCSG:theorem:SZO-Complexity:equality-B}
		B \ge \min \left\{ {n,\frac{1}{{{\varepsilon 
						^2}}}} \right\}.
		\end{align} 
		\item The third term: in order to obtain $8\frac{R}{{SK\eta}}\le 
		\frac{1}{3}{\varepsilon ^2}$,  
		we should 	require the number of iteration,
		\begin{align}
		\label{SCSG-Zero:SCSG:theorem:SZO-Complexity:equality-T}
		\r1 T = SK \ge \frac{{24R}}{{{\varepsilon ^2}\eta }}.
		\end{align}
	\end{itemize}

	Furthermore, denote the total number of $\mathcal{SZO}$: 
	${T_{\mathcal{SZO}}} = 
	S{S_{\mathcal{SZO}}}$,
	where $S = \frac{T}{K}$ is the number of outer iteration, and 
	$S_{\mathcal{SZO}} = DB + K$ is the number of $\mathcal{SZO}$ for each 
	outer 
	iteration. Thus, we have
	\begin{align}
	{T_{\mathcal{SZO}}} = S{S_{\mathcal{SZO}}} =& \frac{T}{K}\left( {DB + 
		K} 
	\right) = T\left( 
	{\frac{DB}{K} + 1} \right)\nonumber\\
	\label{SCSG-Zero:SCSG:theorem:SZO-Complexity:equality-T-1}
	\ge &\frac{R}{{{\varepsilon ^2}\eta }}\left( 
	{\frac{{DB}}{K} + 1} \right)\\
	\label{SCSG-Zero:SCSG:theorem:SZO-Complexity:equality-T-2}\ge & 
	\frac{R}{{{\varepsilon ^2}}}\left( {\frac{1}{{K\eta }}d\eta B + 
		\frac{1}{\eta }} \right)\\
	\label{SCSG-Zero:SCSG:theorem:SZO-Complexity:equality-T-3}\ge& 
	\frac{R}{{{\varepsilon ^2}}}\left( {\max \left\{ {d\eta ,{d^{0.5}}{\eta 
				^{0.5}}} \right\}d\eta B + \frac{1}{\eta }} \right)\\
	\label{SCSG-Zero:SCSG:theorem:SZO-Complexity:equality-T-4}
	\ge& 
	\frac{R}{{{\varepsilon ^2}}}\max \left\{ 
	{{d^{2/3}}{B^{1/3}},{d^{1/3}}{B^{2/3}}} 
	\right\},
	\end{align}
	where
	(\ref{SCSG-Zero:SCSG:theorem:SZO-Complexity:equality-T-1}) is from the 
	bound of 
	$T$ in
	(\ref{SCSG-Zero:SCSG:theorem:SZO-Complexity:equality-T}); 
	(\ref{SCSG-Zero:SCSG:theorem:SZO-Complexity:equality-T-2})
	is from the bound of $D$ in 
	(\ref{SCSG-Zero:Theorem:SZO-Convergence-equality-D}); 
	(\ref{SCSG-Zero:SCSG:theorem:SZO-Complexity:equality-T-3}) 
	is from the bound of $K$ in 
	(\ref{SCSG-Zero:Remark:parameter-u0:equality-1/K}); 	
	(\ref{SCSG-Zero:SCSG:theorem:SZO-Complexity:equality-T-4}) is from 
	following 
	analysis,
	consider the function ${\max \{ {d\eta ,{d^{0.5}}{\eta ^{0.5}}} 
		\}d\eta n + 
		\frac{1}{\eta }}$ and $\eta \le \frac{1}{d^{1/2}}$ in 
	(\ref{SCSG-Zero:Remark:parameter-u0:equality-eta}), we have
	\begin{itemize}
		\item If $\frac{1}{d} \le \eta  \le \frac{1}{{{d^{1/2}}}}$, then
		\begin{align*}
		\max \left\{ {d\eta ,{d^{1/2}}{\eta ^{1/2}}} \right\}d\eta B + 
		\frac{1}{\eta } = \left( {d\eta } \right)d\eta B + \frac{1}{\eta } 
		\ge 
		{d^{2/3}}{B^{1/3}},{\eta ^*} = \frac{1}{{{d^{2/3}}{B^{1/3}}}},
		\end{align*}
		For the difference of $B$,
		\begin{itemize}
			\item $B \le d \Rightarrow \frac{1}{d} \le \eta  = 
			\frac{1}{{{d^{2/3}}{B^{1/3}}}} \le \frac{1}{{{d^{1/2}}}} 
			\Rightarrow 
			\left( {d\eta } \right)d\eta B + \frac{1}{\eta } \ge 
			2{d^{2/3}}B^{1/3}$, 
			\item $B > d \Rightarrow {\eta ^*} = \frac{1}{{{d^{2/3}}{B^{1/3}}}} 
			\le 
			\frac{1}{d} \le \frac{1}{{{d^{1/2}}}}$, then we set $\eta  = 
			\frac{1}{d}$, $\Rightarrow \left( {d\eta } \right)d\eta B + 
			\frac{1}{\eta } \ge 2\left( {B + d} \right)$.
		\end{itemize}
		\item If $\eta  \le \frac{1}{d}$, because $D\ge \eta d$, thus, we set 
		$D=1$, then
		\begin{align*}
		\max \left\{ {d\eta ,{d^{1/2}}{\eta ^{1/2}}} \right\}B + 
		\frac{1}{\eta 
		} = \left( {{d^{1/2}}{\eta ^{1/2}}} \right)B + \frac{1}{\eta } \ge 
		2{d^{1/3}}{B^{2/3}},{\eta ^*} = \frac{1}{{{d^{1/3}}{B^{2/3}}}}.
		\end{align*}
		For the difference of $B$,
		\begin{itemize}
			\item $B \le d \Rightarrow \eta  = \frac{1}{d} \le {\eta ^*} = 
			\frac{1}{{{d^{1/3}}{B^{2/3}}}},\left( {{d^{1/2}}{\eta ^{1/2}}} 
			\right)d\eta B + \frac{1}{\eta } = Bd + d \ge d$;
			\item $B > d \Rightarrow \eta  = {\eta ^*} = 
			\frac{1}{{{d^{1/3}}{B^{2/3}}}} \le \frac{1}{d},\left( 
			{{d^{1/2}}{\eta 
					^{1/2}}} \right)d\eta B + \frac{1}{\eta } = 
			2{d^{1/3}}{B^{2/3}} \le B$.
		\end{itemize}
		Based on the above analysis, we conclude that 
		\begin{align*}
		\max \left\{ {d\eta ,{d^{0.5}}{\eta ^{0.5}}} \right\}d\eta B + 
		\frac{1}{\eta } = \max \left\{ 
		{{d^{2/3}}{B^{1/3}},{d^{1/3}}{B^{2/3}}} 
		\right\},
		\end{align*}
		under the requirement of  $\eta \le \frac{1}{d^{1/2}}$ in 
		(\ref{SCSG-Zero:Remark:parameter-u0:equality-eta}).
	\end{itemize}	 
\end{proof}

\begin{theorem2}\ref{SCSG-Zero:SCSG:theorem:new:SZO-Complexity-Block}  In 
	Algorithm \ref{SCSG-Zero:SCSG:Algorithm:Gaussian:VR-DB-Block}, under 
	Assumption 
	\ref{SCSG-Zero:Assumption-smoothfunctinoVarianceBound}, for $F(x,\xi) 
	\in 
	C_{}^{1,1}$,  
	let the size of  the sample set $\cal B$, $B$ =$O (\min \left\{ 
	{n,1/{\varepsilon 
			^2}} 
	\right\})$, the step 
	$\eta$  =$ O(\min \{ b_0^{1/3}/(d^{2/3}B^{1/3}), b_0^{2/3}/(d^{1/3}B^{2/3}) 
	\})$,  $\mu  \le
	O(\varepsilon /({L_1} d 
	^{1.5}))$, and the 
	number of 
	inner iteration $K \le O (1/\max \{ d\eta ^2,d^{0.5}\eta 
	^{1.5} \})$,   Gaussian vectors set  
	$D \ge O(\eta d)$.
	In 
	order to 
	obtain 
	\begin{align*}
	\frac{1}{S}\sum\nolimits_{s = 0}^{S - 1} {\frac{1}{K}\sum\nolimits_{k = 
			0}^{K - 1} {{{\left\| {\nabla f\left( {{x_k^s}} \right)} 
					\right\|}^2}} }  \le 
	{\varepsilon ^2},
	\end{align*}
	the  total number of $T_{\mathcal{SZO}}$ is at most $O(\max \{ 
	{{d^{2/3}}{B^{1/3}}b_0^{2/3},{d^{1/3}}{B^{2/3}}b_0^{1/3}} \})
	$, with number of total iterations   $T>O(1/(\varepsilon^2 
	\eta))$.
\end{theorem2}
\begin{proof} Based on the results in Theorem 
	\ref{SCSG-Zero:SCSG:Theorem:new:SZO-Convergence}, the former proof is the 
	same 
	as Theorem 
	\ref{SCSG-Zero:SCSG:theorem:new:SZO-Complexity}, the difference lies in the 
	optimal value of $\eta$. The number of $\mathcal{SZO}$ for each 
	outer 
	iteration becomes $S_{\mathcal{SZO}} = DB + Kb_0$, Thus, we have
	\begin{align}
	{T_{\mathcal{SZO}}} = S{S_{\mathcal{SZO}}} =& \frac{T}{K}\left( {DB + 
		Kb_0} 
	\right) = T\left( 
	{\frac{DB}{K} + b_0} \right)\nonumber\\
	\label{SCSG-Zero:SCSG:theorem:SZO-Complexity:equality-T-1-Block}
	\ge &\frac{R}{{{\varepsilon ^2}\eta }}\left( 
	{\frac{{DB}}{K} + b_0} \right)\\
	\label{SCSG-Zero:SCSG:theorem:SZO-Complexity:equality-T-2-Block}\ge & 
	\frac{R}{{{\varepsilon ^2}}}\left( {\frac{1}{{K\eta }}d\eta B + 
		\frac{b_0}{\eta }} \right)\\
	\label{SCSG-Zero:SCSG:theorem:SZO-Complexity:equality-T-3-Block}\ge& 
	\frac{R}{{{\varepsilon ^2}}}\left( {\max \left\{ {d\eta ,{d^{0.5}}{\eta 
				^{0.5}}} \right\}d\eta B + \frac{{{b_0}}}{\eta }} \right)\\
	\label{SCSG-Zero:SCSG:theorem:SZO-Complexity:equality-T-4-Block}
	\ge& 
	\frac{R}{{{\varepsilon ^2}}}\max \left\{ 
	{{d^{2/3}}{B^{1/3}}b_0^{2/3},{d^{1/3}}{B^{2/3}}b_0^{1/3}} \right\},
	\end{align}
	where
	(\ref{SCSG-Zero:SCSG:theorem:SZO-Complexity:equality-T-1-Block}) is from 
	the 
	bound of 
	$T$ in
	(\ref{SCSG-Zero:SCSG:theorem:SZO-Complexity:equality-T}); 
	(\ref{SCSG-Zero:SCSG:theorem:SZO-Complexity:equality-T-2-Block})
	is from the bound of $D$ in 
	(\ref{SCSG-Zero:Theorem:SZO-Convergence-equality-D}); 
	(\ref{SCSG-Zero:SCSG:theorem:SZO-Complexity:equality-T-3-Block}) 
	is from the bound of $K$ in 
	(\ref{SCSG-Zero:Remark:parameter-u0:equality-1/K}); 	
	(\ref{SCSG-Zero:SCSG:theorem:SZO-Complexity:equality-T-4-Block}) is from 
	the 
	following 
	analysis,
	consider the function ${\max \{ {d\eta ,{d^{0.5}}{\eta ^{0.5}}} 
		\}d\eta n + 
		\frac{b_0}{\eta }}$ and $\eta \le \frac{1}{d^{1/2}}$ in 
	(\ref{SCSG-Zero:Remark:parameter-u0:equality-eta}), we have
	\begin{itemize}
		\item If $\frac{1}{d} \le \eta  \le \frac{1}{{{d^{1/2}}}}$, then
		\begin{align*}
		\max \left\{ {d\eta ,{d^{1/2}}{\eta ^{1/2}}} \right\}d\eta B + 
		\frac{{{b_0}}}{\eta } = \left( {d\eta } \right)d\eta B + 
		\frac{{{b_0}}}{\eta } \ge {d^{2/3}}{B^{1/3}}b_0^{2/3},{\eta ^*} = 
		\frac{{b_0^{1/3}}}{{{d^{2/3}}{B^{1/3}}}},
		\end{align*}
		consider  the size of B:
		\begin{itemize}
			\item $\frac{B}{{{b_0}}} \le d \Rightarrow \frac{1}{d} \le \eta  = 
			\frac{{b_0^{1/3}}}{{{d^{2/3}}{B^{1/3}}}} \le \frac{1}{{{d^{1/2}}}} 
			\Rightarrow \left( {d\eta } \right)d\eta n + \frac{{{b_0}}}{\eta } 
			\ge 2{d^{2/3}}{B^{1/3}}b_0^{2/3}$, 
			\item $\frac{B}{{{b_0}}} > d \Rightarrow {\eta ^*} = 
			\frac{{b_0^{1/3}}}{{{d^{2/3}}{B^{1/3}}}} \le \frac{1}{d} \le 
			\frac{1}{{{d^{1/2}}}}$, then we set $\eta  = \frac{1}{d}, 
			\Rightarrow \left( {d\eta } \right)d\eta B + \frac{{{b_0}}}{\eta } 
			\ge 2\left( {B + d{b_0}} \right)$.
		\end{itemize}
		\item If $\eta  \le \frac{1}{d}$,  because $D\ge \eta d$, thus, we set 
		$D=1$, then
		\begin{align*}
		\max \left\{ {d\eta ,{d^{1/2}}{\eta ^{1/2}}} \right\} B + 
		\frac{{{b_0}}}{\eta } = \left( {{d^{1/2}}{\eta ^{1/2}}} \right) B 
		+ \frac{{{b_0}}}{\eta } \ge 2{d^{1/3}}{B^{2/3}}b_0^{1/3},{\eta ^*} 
		= 
		\frac{{b_0^{2/3}}}{{{d^{1/3}}{B^{2/3}}}}.
		\end{align*}
		consider  the size of B:
		\begin{itemize}
			\item $\frac{B}{{{b_0}}} \le d \Rightarrow \eta  = \frac{1}{d} \le 
			{\eta ^*} = \frac{{b_0^{1/3}}}{{{d^{1/3}}{B^{2/3}}}},\left( 
			{{d^{1/2}}{\eta ^{1/2}}} \right)d\eta B + \frac{{{b_0}}}{\eta } = 
			B + db_0 $;
			\item $\frac{B}{{{b_0}}} > d \Rightarrow \eta  = {\eta ^*} = 
			\frac{{b_0^{1/3}}}{{{d^{1/3}}{B^{2/3}}}} \le \frac{1}{d},\left( 
			{{d^{1/2}}{\eta ^{1/2}}} \right)d\eta n + \frac{{{b_0}}}{\eta } = 
			2{d^{1/3}}{B^{2/3}}b_0^{2/3}$.
		\end{itemize}
		Based on the above analysis, we conclude that 
		\begin{align*}
		\max \left\{ {d\eta ,{d^{0.5}}{\eta ^{0.5}}} \right\}d\eta B + 
		\frac{b_0}{\eta } = \max \left\{ 
		{{d^{2/3}}{B^{1/3}}b_0^{2/3},{d^{1/3}}{B^{2/3}}b_0^{1/3}} \right\},
		\end{align*}
		which under the requirement of  $\eta \le \frac{1}{d^{1/2}}$ in 
		(\ref{SCSG-Zero:Remark:parameter-u0:equality-eta}).
	\end{itemize}	 
\end{proof}
\section{Convergence proof for Non-Smooth function with Gaussian smooth}

\subsection{Convergence tool}

In this section, we focus on Algorithm 
\ref{SCSG-Zero:SCSG:Algorithm:Gaussian:VR-DB} 
that 
apply to Gaussian-smoothed function for non-smooth function, and  mainly give 
the upper bounds 
for ${\mathbb{E}_u}{\| {{G_\mu }( {{{\tilde x}_s},{u_{\cal D}},\xi } 
		) - 
		\nabla {f_\mu }( {{{\tilde x}_s}} )} \|^2}$, 
${\mathbb{E}_u}{\left\| {{G_\mu }\left( {{x_k},\xi ,u} \right) - {G_\mu }\left( 
		{{{\tilde x}_s},\xi ,u} \right)} \right\|^2}$ and 
${{\mathbb{E}_{i,j,u}}{\| {{{\tilde \nabla 
					}_k}} \|^2}} $, which are used for 
analyzing the 
convergence sequence. Note, we drop the superscript 
${i}$ and $k$ of $\xi$ and $u$, respectively, for 
focusing on a single epoch analysis.
\begin{lemma}\label{NS-SCSG-Zero:SCSG:Lemma:new:Bound-E[x-Ex]-Onevariance}
	In Algorithm \ref{SCSG-Zero:SCSG:Algorithm:Gaussian:VR-DB}, for $F(x,\xi) 
	\in 
	C^{0,0}$ and ${{G_\mu 
		}\left( {{{ x}_k},{u_{\cal D}},{\xi _{{\mathcal{B}}}}} 
		\right)}$ defined in 
	(\ref{SCSG-Zero:SCSG:Definition-SmoothGradient-Gaussian-mu-block}), we have
	\begin{align*}
	{\mathbb{E}_u}{\left\| {{G_\mu }\left( {{{ x}_k},{u_\mathcal{D}},\xi } 
			\right) - 
			\nabla {f_\mu }\left( {{{ x}_k}} \right)} \right\|^2} \le 
	\frac{1}{D}L_0^2{\left( {d + 2} \right)^2},
	\end{align*}
\end{lemma}
\begin{proof} By the definition of  ${{G_\mu 
		}\left( {{{\tilde x}_s},{u_{\cal D}},\xi_{\mathcal{B}}} 
		\right)}$ defined in 
	(\ref{SCSG-Zero:SCSG:Definition-SmoothGradient-Gaussian-mu-block}), we have
	\begin{align*}
	{\mathbb{E}_u}{\left\| {{G_\mu }\left( {{{ x}_k},{u_\mathcal{D}},{\xi 
					_\mathcal{B}}} \right) 
			- \nabla {f_\mu }\left( {{{ x}_k}} \right)} \right\|^2} \le& 
	{\mathbb{E}_u}{\left\| {{G_\mu }\left( {{{ x}_k},{u_\mathcal{D}},{\xi 
					_\mathcal{B}}} \right)} 
		\right\|^2}\\
	\le& \frac{1}{D}\frac{1}{B}\sum\limits_{i = 1}^B {{\mathbb{E}_u}{{\left\| 
				{{G_\mu }\left( {{{ x}_k},u,{\xi _i}} \right)} 
				\right\|}^2}} \\
	\le& \frac{1}{D}L_0^2{\left( {d + 2} \right)^2},
	\end{align*}
	which is  based on 
	Lemma \ref{SCSG-Zero:Appendix:lemma-expectationSubset} that the vector in 
	$u_{\cal D}$ is independently,  and 
	Lemma \ref{SCSG-Zero:Gaussian:Lemma:Property-fmu}.
\end{proof}

\begin{lemma}\label{NS-SCSG-Zero:SCSG:Lemma:new:Bound-E[G_mu-G_mu]} 
	In Algorithm \ref{SCSG-Zero:SCSG:Algorithm:Gaussian:VR-DB},
	for $F(x,\xi) \in 
	C^{0,0}$, 
	$\mu>0$, and ${{G_\mu }\left( {{x},{u},{\xi 
		}} \right)}$ defined in 
	(\ref{SCSG-Zero:Definition:Gaussian:SmoothGradient}), 
	we have
	\begin{align*}
	{\mathbb{E}_u}{\left\| {{G_\mu }\left( {{x_k},u,\xi } \right) - {G_\mu }\left( 
			{{{\tilde x}_s},u,\xi } \right)} \right\|^2} \le \frac{1}{{{\mu 
				^2}}}\left( {d + 2} \right)L_0^2{\left\| {{x_k} - {{\tilde 
					x}_s}} \right\|^2}.
	\end{align*}
\end{lemma}
\begin{proof} By 
	the definition of ${G_\mu }\left( {{x_k},\xi ,{u}} \right)$ in 
	\ref{SCSG-Zero:Definition:Gaussian:SmoothGradient},  we have
	\begin{align*}
	{\left\| {{G_\mu }\left( {{x_k} ,u,\xi} \right) - {G_\mu }\left( 
			{{{\tilde x}_s} ,u,\xi} \right)} \right\|^2} =& {\left\| 
		{\frac{{F\left( 
					{{x_k} + \mu u,\xi } \right) - F\left( {{x_k},\xi } 
					\right)}}{\mu }u - 
			\frac{{F\left( {{{\tilde x}_s} + \mu u,\xi } \right) - F\left( 
					{{{\tilde x}_s},\xi } \right)}}{\mu }u} \right\|^2}\\
	=& \frac{{{{\left\| u \right\|}^2}}}{{{\mu ^2}}}{\left( {F\left( {{x_k} 
				+ \mu u,\xi } \right) - F\left( {{x_k},\xi } \right) - 
			\left( {F\left( 
				{{{\tilde x}_s} + \mu u,\xi } \right) - F\left( {{{\tilde 
							x}_s},\xi } 
				\right)} \right)} \right)^2}\\
	\le& \frac{{{{\left\| u \right\|}^2}}}{{{\mu ^2}}}L_0^2{\left\| {{x_k} 
			- 
			{{\tilde x}_s}} \right\|^2},
	\end{align*}
	where  the last inequality follows from 
	Lipschitz continue of $F(x,\xi)$.	Take expectation with respect to 
	$u$, we have
	\begin{align*}
	{\mathbb{E}_u}{\left\| {{G_\mu }\left( {{x_k},u,\xi } \right) - {G_\mu }\left( 
			{{{\tilde x}_s},u,\xi } \right)} \right\|^2} \le \frac{1}{{{\mu 
				^2}}}\left( {d + 2} \right)L_0^2{\left\| {{x_k} - {{\tilde 
					x}_s}} 
		\right\|^2},
	\end{align*}
	where the second inequality is based on Lemma 
	\ref{SCSG-Zero:Appendix:P-Bound-Gaussian-distribution} for $p=2$.
\end{proof}
Similar to Lemma \ref{SCSG-Zero:SCSG:Lemma:new:Bound-Gradient:variance}, 
based on Lemma \ref{NS-SCSG-Zero:SCSG:Lemma:new:Bound-E[G_mu-G_mu]} and Lemma 
\ref{NS-SCSG-Zero:SCSG:Lemma:new:Bound-E[x-Ex]-Onevariance}, we have

\begin{lemma}\label{NS-SCSG-Zero:SCSG:Lemma:new:Bound-Gradient:variance} 
	In Algorithm \ref{SCSG-Zero:SCSG:Algorithm:Gaussian:VR-DB}, for $F(x,\xi) 
	\in 
	C_{}^{1,1}$, 
	$\mu>0$, and ${{\tilde \nabla 
		}_k}$ defined in 
	(\ref{SCSG-Zero:SCSG:Definition-EstimateGradient-Gaussian}),   we have
	\begin{align*}
	{{\mathbb{E}_{i,j,u}}{{\| {{{\tilde \nabla 
						}_k}} \|}^2}}  
	\le&\frac{1}{{{\mu 
				^2}}}\left( {d + 2} \right)L_0^2{\left\| {{x_k} - {{\tilde 
					x}_s}} \right\|^2}+ 3 
	{{{\left\| 
				{\nabla {f_\mu }\left( {{x_k}} \right)} \right\|}^2}}
	+\frac{3}{D}L_0^2{\left( {d + 2} \right)^2}.
	\end{align*}
\end{lemma}

\subsection{Convergence analysis}

In this subsection, mainly based on  Lemma 
\ref{SCSG-Zero:SCSG:Lemma:new:Bound-Gradient:variance}, smoothness and update 
of 
$x$ in 
Algorithm \ref{SCSG-Zero:SCSG:Algorithm:Gaussian:VR-DB}, we give the new 
sequence of the proposed algorithm: ${\mathbb{E}_{i,j}}\left[ {{f_\mu }\left( 
	{{x_{k + 1}}} \right)} 
\right] + {c_{k 
		+ 1}}{\mathbb{E}_{i,j}}{\left\| {{x_{k + 1}} - {{\tilde x}_s}} 
	\right\|^2}$. In order to obtain the convergence sequence,we provide the 
formulation of the 
sequence $c_k$, $w_k$ and $J_k$, which is the key parameter in analyzing the 
convergence and complexity.  In 
Remark \ref{SCSG-Zero:Remark:parameter-C} and  
\ref{SCSG-Zero:Remark:parameter-u0}, we 
analyze the the parameter's 
relationship between $K$, $q$ and $\eta$ such that these new formed 
sequence can be converged.

\begin{lemma}\label{NS-SCSG-Zero:SCSG:Lemma:new:sum-SK}
	In Algorithm 
	\ref{SCSG-Zero:SCSG:Algorithm:Gaussian:VR-DB},  
	for $F(x,\xi) \in 
	C_{}^{0,0}$, $\mu>0$, $q>0$, we have
	\begin{align*}
	\frac{1}{S}\sum\limits_{s = 0}^{S-1} {\frac{1}{K}\sum\limits_{k = 0}^{K - 
			1} 
		{{{ w}_k}{{\left\| {\nabla f_{\mu}\left( {x_k^s} \right)} 
					\right\|}^2}} }  
	\le \frac{R}{{SK}} + {J_{k + 1}},
	\end{align*}
	where   $x^*$ is the 
	optimal value of function $f_{\mu}(x)$, $R = {\max _x}\{ f_{\mu}( x ) - 
	f{_\mu}( x_* 
	):f{_\mu}( x ) \le f{_\mu}( x_0 ) \}$, and 
	\begin{align}
	\label{NS-SCSG-Zero:Lemma:Bound:f+cx-ck}
	{c_k} =& \left( {1 + \eta q + \frac{1}{{{\mu ^2}}}\left( {d + 2} 
		\right)L_0^2{\eta ^2}} \right){c_{k + 1}} + \frac{{{L_1}{\eta 
				^2}}}{2}\frac{1}{{{\mu ^2}}}\left( {d + 2} \right)L_0^2,\\
	\label{NS-SCSG-Zero:Lemma:Bound:f+cx-Jk}
	{w_k} =& \left( {1 - \frac{1}{q}{c_{k + 1}}} \right)\eta  - \left( 
	{\frac{{3{L_1}}}{2} + {c_{k + 1}}} \right){\eta ^2},\\
	\label{NS-SCSG-Zero:Lemma:Bound:f+cx-uk}	
	{{ J}_k} =& \frac{{{L_1}{\eta ^2}}}{2}\frac{3}{D}L_0^2{\left( {d + 2} 
		\right)^2} + {c_k}{\eta ^2}\frac{3}{D}L_0^2{\left( {d + 2} 
		\right)^2}.	
	\end{align}
\end{lemma}

\subsubsection{Parameters analysis}

\begin{remark}\label{SCSG-Zero:Nonsmooth:Remark-parameter-mu}
	For $\mu$. The relationship between  the non-smooth function and 
	the 
	smoothed function is $\left| {{f_\mu }\left( x 
		\right) - f\left( x \right)} \right| 
	\le \mu {L_0}{d^{1/2}}$ (Theorem 1 \cite{nesterov2017random}, Lemma 
	\ref{SCSG-Zero:Gaussian:Lemma:Property-fmu}).  In 
	order 
	to bound the gap in 
	approximation by $\varepsilon$, we follow the setting as in 
	\cite{nesterov2017random},   and set 
	\begin{align}\label{SCSG-Zero:Nonsmooth:Remark:equality-mu}
	 \mu \le
	\frac{1}{{{L_0}{d^{1/2}}}}\varepsilon, 
	\end{align}
	such that $\left| {{f_\mu }\left( x 
		\right) - f\left( x \right)} \right| 
	\le \varepsilon$.	
\end{remark}

\begin{remark}\label{SCSG-Zero:Nonsmooth:Remark-parameter-all}
	Based on Lemma \ref{SCSG-Zero:Gaussian:Lemma:Property-fmu}, we have 
	$L_1(f_\mu ) 
	= 
	d^{1/2}L_0/\mu $, then we consider the sequence $\left\{ {{J_k}} 
	\right\}$, $\left\{ {{u_k}} 
	\right\}$ and $\left\{ {{c_k}} \right\}$, and delete the non-related 
	coefficient that do not affect the convergence, and remain 
	$d,\eta,\varepsilon,q$, then
	\begin{align*}
	{c_k} =& \left( {1 + \eta q + \frac{{{d^2}}}{{{\varepsilon ^2}}}{\eta 
			^2}} 
	\right){c_{k + 1}} + \frac{{{d^3}}}{{{\varepsilon ^2}}}{\eta ^2};\\
	{w_0} =& \left( {1 - \frac{1}{q}{c_0}} \right)\eta  - \left( 
	{\frac{d}{\varepsilon } + {c_0}} \right){\eta ^2};\\
	{J_0} =& \frac{1}{D}\frac{{{d^3}}}{\varepsilon }{\eta ^2} + 
	\frac{1}{D}{d^2}{c_0}{\eta ^2}+\left( {1 + \frac{1}{q}{c_{k + 1}}} 
	\right)\frac{1}{2}\eta \frac{{\mathbb{I}(B < n)}}{B}H.
	\end{align*}
	Based on the analysis in Theorem 
	\ref{SCSG-Zero:SCSG:theorem:new:SZO-Complexity}, Remark 
	\ref{SCSG-Zero:Remark:parameter-u0} and  
	Remark\ref{SCSG-Zero:Remark:parameter-C}, we conclude and require that
	\begin{align}
	\label{SCSG-Zero:Nonsmooth:Remark-parameter-all-1}
	q =& \frac{{{d^{3/2}}}}{\varepsilon }{\eta ^{1/2}},\\
	\label{SCSG-Zero:Nonsmooth:Remark-parameter-all-2}
	{w_0} =& \eta,{c_0} < d,\\
	\label{SCSG-Zero:Nonsmooth:Remark-parameter-all-3}
	1 >& \eta q + \frac{{{d^2}}}{{{\varepsilon ^2}}}{\eta ^2} \Rightarrow 
	\eta  
	< \frac{\varepsilon }{d},\\
	\label{SCSG-Zero:Nonsmooth:Remark-parameter-all-4}
	K \le& 1/\max \left\{ {\frac{{{d^{3/2}}}}{\varepsilon }{\eta 
			^{2/3}},\frac{{{d^2}}}{{{\varepsilon ^2}}}{\eta ^2}} 
	\right\}=\frac{\varepsilon^2}{d^2\eta^2},\\
	\label{SCSG-Zero:Nonsmooth:Remark-parameter-all-5}
	D \ge& \max \left\{ {\frac{{{d^3}}}{{{\varepsilon ^3}}}{\eta 
		},\frac{1}{{{\varepsilon ^2}}}{d^3}{\eta}} 
	\right\}=\frac{{{d^3}}}{{{\varepsilon ^3}}}{\eta 
	},\\
	\label{SCSG-Zero:Nonsmooth:Remark-parameter-all-6}
	B =& \min \left\{ {n,\frac{1}{{{\varepsilon ^2}}}} \right\}.
	\end{align}
	
\end{remark}

\subsection{Complexity analysis}

\begin{theorem2}\ref{NS-SCSG-Zero:SCSG:theorem:new:SZO-Complexity} In 
	Algorithm 
	\ref{SCSG-Zero:SCSG:Algorithm:Gaussian:VR-DB}, for $F(x,\xi) \in 
	C^{0,0}$,    the step 
	$\eta  = O({\varepsilon ^{5/3}}/( d^{5/3}B^{1/3} ))$,  $\mu  
	\le
	O(\varepsilon /({L_0} d 
	^{1/2}))$, and the 
	number of 
	inner iteration $K \le O(\varepsilon^2/(d^2\eta^2))$,   Gaussian vectors 
	set  
	$D \ge O(\eta d^3/\varepsilon^3)$.
	In 
	order to 
	obtain 
	\begin{align*}
	\frac{1}{S}\sum\nolimits_{s = 0}^{S - 1} {\frac{1}{K}\sum\nolimits_{k = 
			0}^{K - 1} {{{\left\| {\nabla f_{\mu}\left( {{x_k^s}} \right)} 
					\right\|}^2}} }  \le 
	{\varepsilon ^2},
	\end{align*}
	the  total number of $T_{\mathcal{SZO}}$ is 
	$O({{{d^{5/3}}{B^{1/3}}}}/{{{\varepsilon ^{11/3}}}})$, {number of inner 
		iterations}  
	$T>O(1/(\varepsilon^2 
	\eta))$.
\end{theorem2}
\begin{proof} Based on the results in Theorem 
	\ref{SCSG-Zero:SCSG:Theorem:new:SZO-Convergence}, in order 
	to obtain 
	\begin{align*}
	\frac{R}{{SK}} + {J_{k + 1}} 
	\le 
	\varepsilon^2,
	\end{align*}
	we separately  analysis to obtain the complexity:
	
	Furthermore, denote the total number of $\mathcal{SZO}$: 
	${T_{\mathcal{SZO}}} = 
	S{S_\mathcal{SZO}}$,
	where $S = \frac{T}{K}$ is the number of outer iteration, and 
	$S_{\mathcal{SZO}} = DB + K$ is the number of $\mathcal{SZO}$ for each 
	outer 
	iteration. Thus, we have
	\begin{align}
	{T_{\mathcal{SZO}}} = S{S_{\mathcal{SZO}}} =& \frac{T}{K}\left( {DB + 
		K} 
	\right) = T\left( 
	{\frac{DB}{K} + 1} \right)\nonumber\\
	\label{NS-SCSG-Zero:SCSG:theorem:SZO-Complexity:equality-T-1}
	\ge &\frac{R}{{{\varepsilon ^2}\eta }}\left( 
	{\frac{{DB}}{K} + 1} \right)\\
	= & 
	\frac{R}{{{\varepsilon ^2}}}\left( {\frac{{DB}}{{K\eta }} + 
		\frac{1}{\eta }} \right)\nonumber\\
	\label{NS-SCSG-Zero:SCSG:theorem:SZO-Complexity:equality-T-3}\ge& 
	\frac{R}{{{\varepsilon ^2}}}\left( {\frac{{{d^2}\eta 
		}}{{{\varepsilon ^2}}}\frac{{{d^3}}}{{{\varepsilon ^3}}}\eta B + 
		\frac{1}{\eta }} \right) \Rightarrow {\eta ^*} = 
	\frac{{{\varepsilon 
				^{5/3}}}}{{{d^{5/3}}{B^{1/3}}}}\\
	\label{NS-SCSG-Zero:SCSG:theorem:SZO-Complexity:equality-T-4}
	\ge& \frac{R}{{{\varepsilon 
				^2}}}\frac{{{d^{5/3}}{B^{1/3}}}}{{{\varepsilon ^{5/3}}}} = 
	\frac{{{d^{5/3}}{B^{1/3}}}}{{{\varepsilon ^{11/3}}}}R,
	\end{align}
	where
	(\ref{NS-SCSG-Zero:SCSG:theorem:SZO-Complexity:equality-T-1}) is from the 
	bound of 
	$T$ in
	(\ref{SCSG-Zero:SCSG:theorem:SZO-Complexity:equality-T}); 
	(\ref{NS-SCSG-Zero:SCSG:theorem:SZO-Complexity:equality-T-3}) 
	is from the bound of $K$ in 
	(\ref{SCSG-Zero:Nonsmooth:Remark-parameter-all-4}) and $D$ in 
	(\ref{SCSG-Zero:Nonsmooth:Remark-parameter-all-5});	
	(\ref{NS-SCSG-Zero:SCSG:theorem:SZO-Complexity:equality-T-4}) is from 
	the optimal value of $\eta^*$.	
\end{proof}

If for mini-batch SZVR, we can obtain the results from Theorem 
\ref{SCSG-Zero:SCSG:theorem:new:SZO-Complexity-Block} and Theorem 
\ref{NS-SCSG-Zero:SCSG:theorem:new:SZO-Complexity}.
\begin{theorem}\label{NS-SCSG-Zero:SCSG:theorem:new:SZO-Complexity-Block} In 
	Algorithm 
	\ref{SCSG-Zero:SCSG:Algorithm:Gaussian:VR-DB}, for $F(x,\xi) \in 
	C_{}^{0,0}$,    the step 
	$\eta  = O({\varepsilon ^{5/3}}b_0^{1/3}/( {{d^{5/3}}{B^{1/3}}} 
	))$,  $\mu  
	\le
	O(\varepsilon /({L_0} d 
	^{1/2}))$, and the 
	number of 
	inner iteration $K \le O(\varepsilon^2/(d^2\eta^2))$,   Gaussian vectors 
	set  
	$D \ge O(\eta d^3/\varepsilon^3)$.
	In 
	order to 
	obtain 
	\begin{align*}
	\frac{1}{S}\sum\nolimits_{s = 0}^{S - 1} {\frac{1}{K}\sum\nolimits_{k = 
			0}^{K - 1} {{{\left\| {\nabla f_{\mu}\left( {{x_k^s}} \right)} 
					\right\|}^2}} }  \le 
	{\varepsilon ^2},
	\end{align*}
	the  total number of $T_{\mathcal{SZO}}$ is 
	$O({{{d^{5/3}}{B^{1/3}}}}/{{({\varepsilon ^{11/3}}b_0^{1/3})}})$,  
	iteration  
	$T>O(1/(\varepsilon^2 
	\eta))$.
\end{theorem}

%%%%%%%%%%%%%%%%%%%%%
%%%%%%%%%%%%%%%%%%%%
%%%%%%%%%%%%%%%%%%%%%
%%%%%%%%%%%%%%%%%%%%%%
%%%%%%%%%%%%%%%%%%%%%%%%

\end{document}